\documentclass{article}

 \usepackage[final,nonatbib]{nips_2016}

\usepackage[utf8]{inputenc} 
\usepackage[T1]{fontenc}    
\usepackage{hyperref}       
\usepackage{url}            
\usepackage{booktabs}       
\usepackage{amsfonts}       
\usepackage{nicefrac}       
\usepackage{microtype}      

\usepackage{amsmath}
\usepackage{amsthm}
\usepackage{amssymb}
\usepackage{algorithm}
\usepackage{algorithmic}
\usepackage{graphicx}
\usepackage{subcaption}
\usepackage{rotating}
\newtheorem{thm}{Theorem}
\newtheorem{lm}{Lemma}
\newtheorem{claim}{Claim}
\newtheorem{defn}{Definition}

\newtheorem{prop}{Proposition}

\title{Convergence rate of stochastic $k$-means}

\author{
  Cheng Tang\\
  Department of Computer Science\\
  George Washington University\\
  Washington, DC 20052 \\
  \texttt{tangch@gwu.edu} \\
  \And
Claire Monteleoni \\
  Department of Computer Science\\
  George Washington University\\
  Washington, DC 20052 \\
  \texttt{cmontel@gwu.edu} \\
}

\begin{document}

\maketitle

\begin{abstract}
We analyze online \cite{BottouBengio} and mini-batch \cite{Sculley} $k$-means variants. Both scale up the widely used $k$-means algorithm via stochastic approximation, and have become popular for large-scale clustering and unsupervised feature learning.
We show, for the first time, that starting with any initial solution, they converge to a ``local optimum'' at rate $O(\frac{1}{t})$ (in terms of the $k$-means objective) under general conditions.  
In addition, we show if the dataset is clusterable, when initialized with a simple and scalable seeding algorithm, mini-batch $k$-means converges to an optimal $k$-means solution at rate $O(\frac{1}{t})$ with high probability. 
The $k$-means objective is non-convex and non-differentiable: we exploit ideas from recent work on stochastic gradient descent for non-convex problems \cite{ge:sgd_tensor, balsubramani13} by providing a novel characterization of the trajectory of $k$-means algorithm on its solution space, and circumvent the non-differentiability problem via geometric insights about $k$-means update.
\end{abstract}
\section{Introduction}
\label{sec:intro}
Lloyd's algorithm (batch $k$-means) \cite{lloyd} is one of the most popular heuristics for clustering \cite{jain:kmeans}. 
However, at every iteration it requires computation of the closest centroid to every point in the dataset. Even with fast implementations such as \cite{elkan:kmeans_fast}, which reduces the computation for finding the closest centroid of each point, the per-iteration running time still depends linearly on $n$, making it a computational bottleneck for large datasets.  
 To scale up the centroid-update phase, a plausible recipe is the ``stochastic approximation'' scheme \cite{bottou:SA}: the overall idea is, at each iteration, the centroids are updated using one (online \cite{BottouBengio}) or a few (mini-batch \cite{Sculley}) randomly sampled points instead of the entire dataset. In the rest of the paper, we refer to both as stochastic $k$-means, which we formally present as Algorithm \ref{alg:MBKM}. 
Empirically, stochastic $k$-means has gained increasing attention for large-scale clustering and is included in widely used machine learning packages, such as \texttt{Sofia-ML} \cite{Sculley} and \texttt{scikit-learn} \cite{scikit-learn}.
Figure \ref{fig:demo} demonstrates the efficiency of stochastic $k$-means against batch $k$-means on the \texttt{RCV1} dataset \cite{data:rcv1}. The advantage is clear, and the results raise some natural questions:
Can we characterize the convergence rate of stochastic $k$-means?
Why do the algorithms appear to converge to different ``local optima''?
Why and how does mini-batch size affect the quality of the final solution?
Our goal is to address these questions rigorously.

Our main contribution\footnote{
An improved and better (according to the authors' opinion) version of this paper since the NIPS submission 
is available at 
\url{https://arxiv.org/abs/1610.04900}
} is the $O(\frac{1}{t})$ global convergence of stochastic $k$-means. Our key idea is to keep track of the distance from the algorithm's current clustering solution to the set of all ``local optima'' of batch $k$-means: when the distance is large with respect to all local optima, the drop in $k$-means objective after one iteration of stochastic $k$-means update is lower bounded; when the algorithm is close enough to a local optimum, the algorithm will be trapped by it and maintains the $O(\frac{1}{t})$ convergence rate thanks to a geometric property of local optima.
\begin{figure}
\begin{subfigure}{0.46\textwidth}
\includegraphics[height = \linewidth, width = 0.6\linewidth, angle = -90]{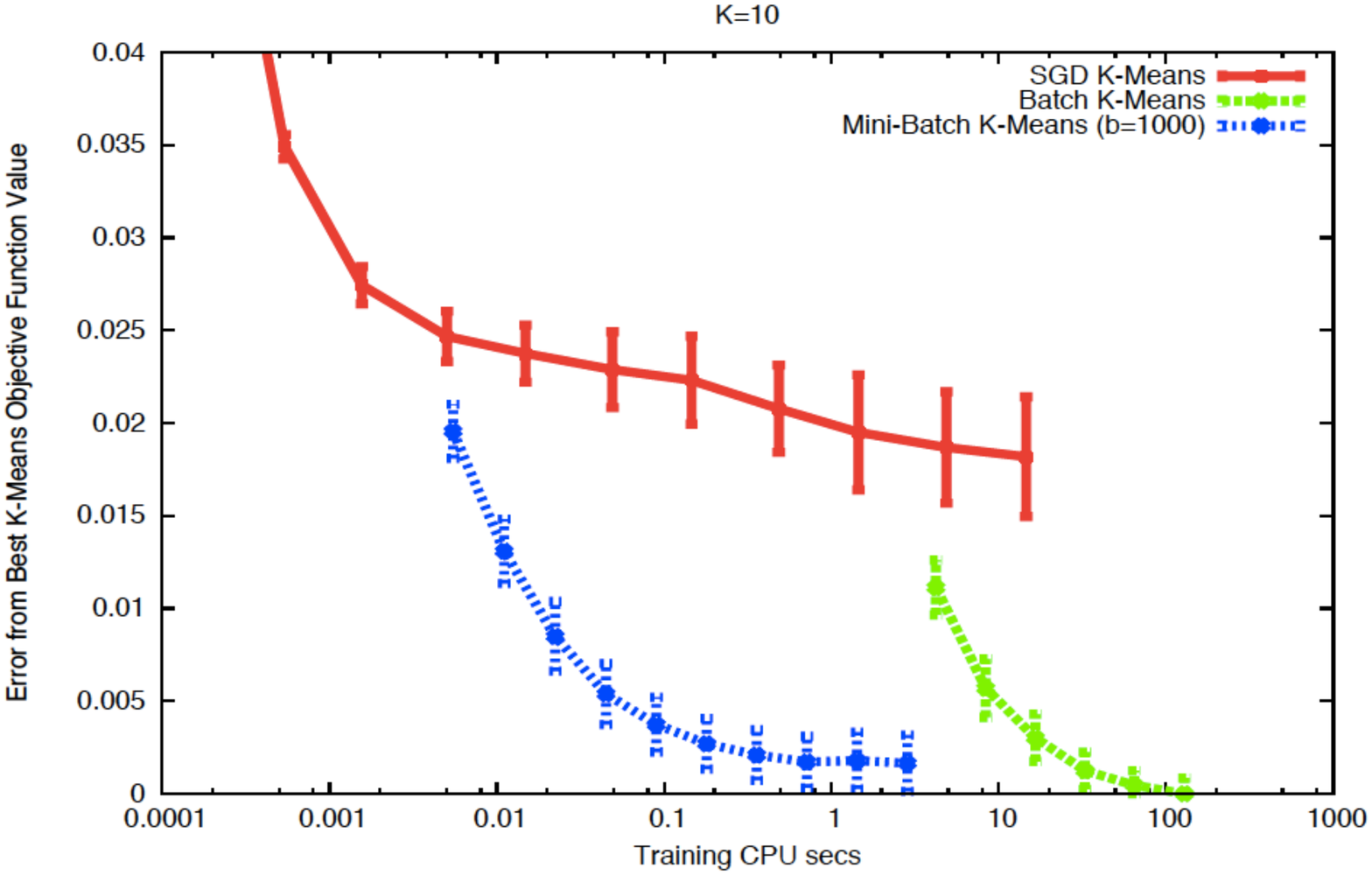}
\caption{Figure from \cite{Sculley}, demonstrating the relative performance
of online, mini-batch, and batch $k$-means.
}
\label{fig:demo}
\end{subfigure}\hfill
\begin{subfigure}{0.5\textwidth}
\includegraphics[height=\linewidth, width = 0.55\linewidth, angle = -90]{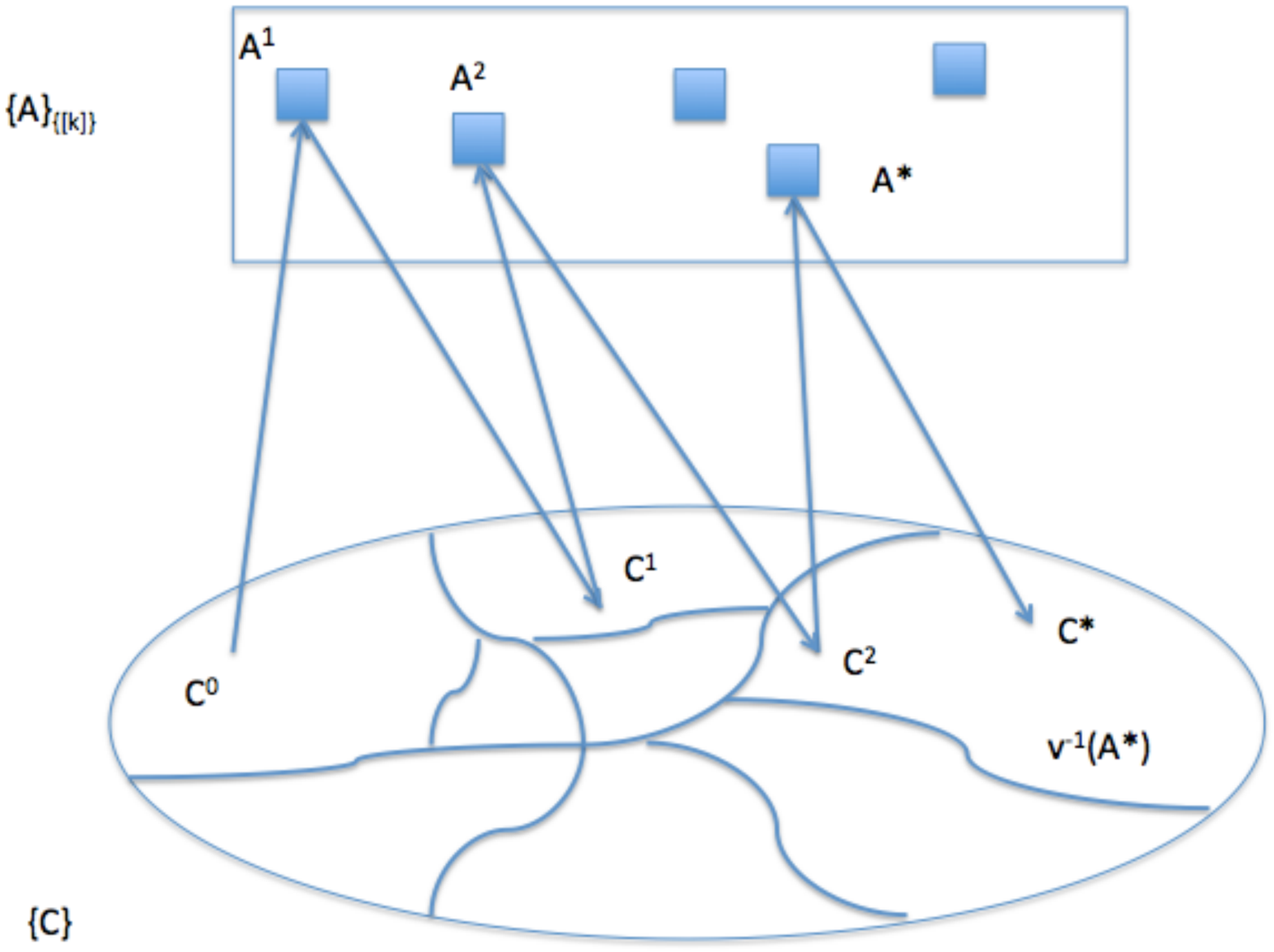}
\caption{An illustration of one run of Lloyd's algorithm:
the arrows represent k-means updates in action.
}
\label{fig:solution_spaces}
\end{subfigure}
\caption{}
\vspace{-0.5cm}
\end{figure}
\paragraph{Notation}
The input to our clustering problem is a discrete dataset $X$; $\forall x\in X$, $x\in \mathbb{R}^d$. 
We use letter $A$ to denote clusterings of $X$;
we use letter $C$ to denote a set of $k$-centroids.
Superscripts are used to indicate a particular clustering, e.g., $A^t$ denotes the clustering at the $t$-th iteration in Algorithm \ref{alg:MBKM} (or batch $k$-means);
subscripts indicate individual clusters in a clustering: $c_r$ denotes the $r$-th centroid in $C$ corresponding to the $r$-th cluster $A_r$.
We use letter $n$ to denote cardinality, $n = |X|$, $n_r=|A_r|$, etc.
Fix a point set $Y$, we let $m(Y)$ denote the mean of $Y$. Each clustering $A:=\{A_s,s\in [k]\}$ induces a unique set of centroids $m(A):=\{m(A_s), s\in[k]\}$. 
Fix $C=\{c_r,r\in [k]\}$, we let $V(c_r)$ denote the Voronoi cell belonging to centroid $c_r$, i.e., $\{x\in \mathbb{R}^d, \|x-c_r\|\le\|x-c_s\|,\forall s\ne r\}$, and we use $V(C)$ to denote the Voronoi diagram induced by a set of centroids $C$.
Fix $C$ with $k$-centroids, we denote its $k$-means cost with respect to a $k$-clustering $A$ by
$\phi(C,A):=\sum_{r=1}^{k}\sum_{x\in A_r}\|x-c_r\|^2$ 
(or simply $\phi(C)$ when $A$ is induced by the Voronoi diagram of $C$).
We let $\phi^{*}:=\phi(C^*)$, $\phi^t:=\phi(C^t)$, and let $\phi^{*}_r$ ($\phi^t_r$) denote the $k$-means cost of cluster $A^*_r$ ($A^t_r$). $\phi^{opt}$ and $\phi^{opt}_r$ are similarly defined for the cost of an optimal $k$-means clustering.
We use $\pi: [k]\rightarrow [k]$ to denote permutations between two sets of the same cardinality.
\section{A framework for tracking batch $k$-means in the solution space}
\label{sec:framework}
First, we develop insights for batch $k$-means to facilitate our analysis of stochastic $k$-means.\footnote{Due to space limit, we move the detailed construction of the framework to Appendix A, where more formal definitions and omitted discussions of issues such as degenerate cases can be found.} 
Batch $k$-means initializes the position of $k$ centroids $C^0$ via a seeding algorithm. Then $\forall t\ge 1$,
\begin{enumerate}
\vspace{-0.2cm}
\item
Obtain $A^{t}$ by assigning each $x$ to its closest centroid (Voronoi cell).
\item
For all $A_r^t$ that is not empty, obtain $c_r^t:=m(A_r^t)$.
\vspace{-0.2cm}
\end{enumerate}
The algorithm alternates between two solution spaces: the continuous space of all $k$-centroids, which we denote by $\{C\}$, and the finite set of all $k$-clusterings, which we denote by $\{A\}$. One problem with $k$-means is it may produce degenerate solutions: if the solution $C^t$ has $k$ centroids, it is possible that data points are mapped to only $k^{\prime}<k$ centroids. To handle degenerate cases, starting with $|C^0|=k$, we
consider an enlarged clustering space $\{A\}_{[k]}$, which is the union of all $k^{\prime}$-clusterings 
with $1\le k^{\prime}\le k$. 
Our key idea is that $\{C\}$ can be partitioned into equivalence classes by the clustering they induce on $X$, and the algorithm stops if and only if two consecutive iterations stay within the same equivalence class in \{C\}. 
Specifically, for any $C$, let 
$
v(C):= V(C)\cap X \in \{A\}_{[k]}
$ denote the clustering they induce on $X$ via their Voronoi diagram.
Then $C_1,C_2$ are in the same equivalence class in $\{C\}$ if $v(C_1)=v(C_2)=A$.
In lieu of this construction, the algorithm goes from $\{C\}$ to $\{A\}_{[k]}$ via mapping $v:\{C\}\rightarrow \{A\}_{[k]}$;
it goes from $\{A\}_{[k]}$ to $\{C\}$ via the mean operation $m:\{A\}_{[k]}\rightarrow \{C\}$.
Figure \ref{fig:solution_spaces} illustrates how batch $k$-means alternates between
two solution spaces until convergence.
We use the pre-image $v^{-1}(A)\in \{C\}$ to denote the equivalence class induced by clustering $A$. When batch $k$-means visits a clustering $A^t$, if $m(A^t)\notin v^{-1}(A^t)$, the algorithms jumps to another clustering $A^{t+1}$.
If $m(A^t)\in v^{-1}(A^t)$, the algorithm stops because $A^{t+1}=A^t$ and $m(A^{t+1})=m(A^t)$. 
We thus formalize the idea of ``local optima'' of batch $k$-means as below.
 \begin{defn}[Stationary clusterings]
We call $A^*$ a stationary clustering of $X$, if $m(A^*)\in Cl(v^{-1}(A^*))$.
We let $\{A^*\}_{[k]}\subset \{A\}_{[k]}$ denote the set of all stationary clusterings of $X$
with number of clusters $k^{\prime}\in [k]$.
\end{defn}
The operator $Cl(\cdot)$ denotes the ``closure'' of an equivalence class $v^{-1}(A^*)$, which includes its ``boundary points'', a set of centroids that induces ambiguous clusterings (this happens when there is a data point on the bisector of two centroids in a solution; see Appendix A for details). For each $A^*$, we define a matching centroidal solution $C^*$.
\begin{defn}[Stationary points]
For a stationary clustering $A^*$ with $k^{\prime}$ clusters, we define $C^{*}=\{c_r^*, r\in [k^{\prime}]\}$ to be a stationary point corresponding to $A^*$, so that $\forall A_{r}^{*}\in A^{*}$, $c_r^*:=m(A_r^{*})$.  
We let $\{C^*\}_{[k]}$ denote the corresponding set of all stationary points of $X$ with $k^{\prime}\in [k]$.
\end{defn}
As preluded in the introduction, defining a distance measure on $\{C\}$ is important to our subsequent analysis; For $C^{\prime}$ and $C$, we let
$\Delta(C^{\prime},C):=\min_{\pi: [k]\rightarrow [k]}\sum_r n_r\|c^{\prime}_{\pi(r)}-c_{r}\|^2$, where $n_r=|A_r|$. 
This distance is asymmetric and non-negative, and evaluates to zero if and only if two sets of centroids coincide. In addition, if $C^*$ is a stationary point, then for any solution $C$, $\Delta(C,C^*)$ upper bounds the difference of $k$-means objective, $\phi(C)-\phi(C^*)$ (Lemma \ref{lm:kmdist_cdist}).
\subsection{Stochastic $k$-means}
Algorithm \ref{alg:MBKM} and stochastic $k$-means \cite{BottouBengio,Sculley} are equivalent up to the choice of learning rate and sampling scheme (the proof of equivalence is in Appendix A).
In \cite{BottouBengio,Sculley}, the per-cluster learning rate is chosen as 
$\eta^t_r:=\frac{\hat{n}_r^t}{\sum_{i\le t}\hat{n}_r^i}$; in our analysis, we choose a flat learning rate $\eta^t=\frac{c^{\prime}}{t_o+t}$ for all clusters, where $c^{\prime}, t_o>0$ are some fixed constants (empirically, no obvious differences are observed; see Section \ref{sec:related_work} for more discussion and Section \ref{sec:EXP} for empirical comparison).
\begin{algorithm}[t]
   \caption{Stochastic $k$-means}
   \label{alg:MBKM}
\begin{algorithmic}
   \STATE {\bfseries Input:} dataset $X$, number of clusters $k$, mini-batch size $m$, learning rate $\eta_r^t, r\in [k]$,
     \texttt{convergence\_criterion}
   \STATE{\bfseries Seeding:} Apply seeding algorithm $\mathcal{T}$ on $X$ and obtain seeds $C^0=\{c_1^0,\dots,c_k^0\}$; 
   \REPEAT
   \STATE At iteration $t$ ($t\ge 1$), obtain sample $S^t\subset X$ of size $m$ uniformly at random with replacement;
     set count $\hat{n}^t_r \leftarrow 0$ and set $S^t_{r}\leftarrow \emptyset$, $\forall r\in [k]$
   \FOR{$s\in S^t$} 
   	   \STATE Find $I(s)$ s.t. $c_{I(s)}=C(s)$
        \STATE $S^t_{I(s)}\leftarrow S^t_{I(s)}\cup s$; $\hat{n}_{I(s)}^t \leftarrow \hat{n}_{I(s)}^t + 1$
   \ENDFOR
   \FOR{$c_r^{t-1}\in C^{t-1}$}
   \IF{$\hat{n}_r^t\ne 0$}
     \STATE $c_r^t \leftarrow (1-\eta^t_r)c_r^{t-1} + \eta^t_r\hat{c}_r^t$ with 
     $\hat{c}_r^t:=\frac{\sum_{s\in S^t_r} s}{\hat{n}_r^t}$
   \ENDIF
   \ENDFOR
   \UNTIL{\texttt{convergence\_criterion} is satisfied}
\end{algorithmic}
\end{algorithm}
Unlike a usual gradient-based algorithm on a continuous domain, the discrete nature of the $k$-means problem causes major differences when stochastic approximation is applied. 

First, for batch $k$-means, at every iteration $t$, $C^{t}$ is chosen as the means, $m(A^{t})$, of the current $k$-clustering $A^t$. Since $\{A\}_{[k]}$ is finite and $m(A)$ is unique for a fixed $A$, the set of ``legal moves'' of batch $k$-means is finite, while $\{C\}$ is continuous.
In stochastic $k$-means, however, $C^t$ can be any member of $\{C\}$ due to both stochasticity and the learning rate. As such, its effective solution space on $\{C\}$ is continuous and thus infinite. Our framework enables us to impose a finite structure on $\{C\}$ by mapping points in $\{C\}$ to points in $\{A\}_{[k]}$.

Second, in stochastic $k$-means, centroids are usually updated asynchronously, especially when the mini-batch size is small compared to $k$; in the extreme case of online $k$-means, centroids are updated one at a time.
Since the positions of centroids implicitly determine the clustering assignment, asynchronous $k$-means updates will lead to a different solution path than fully synchronized ones, even if we ignore the noise introduced by sampling and choose the learning rate to be 1.
Due to asynchronicity, it is also hard to detect when stochastic $k$-means produces a degenerate solution, since centroids may fail to be updated for a long time due to sampling. In practice, implementation such as \texttt{scikit-learn} re-locate a centroid randomly if it fails to be updated for too long. In Algorithm \ref{alg:MBKM}, we do not allow re-location of centroids and our analysis subsumes the degenerate cases.

When analyzing Algorithm \ref{alg:MBKM}, we let $\Omega$ denote the sample space of all outcomes $(C^0,C^1,\hdots,C^t,\hdots)$ and $F_t$ the natural filtration of the stochastic process up to $t$ ($\Omega$ is also used in our statements of Theorems \ref{thm:km} and \ref{thm:solution} as the Big-\textit{Omega}).
We let
$
p_r^t(m):=Pr\{c_r^{t-1}\mbox{~is updated at~}t\mbox{~with mini-batch size~}m\}
$.
\section{Main result}
With two weak assumptions on the properties of stationary points of a dataset $X$, our first result shows stochastic $k$-means has an overall $O(\frac{1}{t})$ convergence rate to a stationary point of batch $k$-means, regardless of initialization.
\paragraph{(A0)}
We assume all stationary points are non-boundary points, i.e., $\forall A^*\in \{A^*\}_{[k]}$, $m(A^*)\in v^{-1}(A^*)$
(By Lemma \ref{lm:stat_stab}, under this assumption, $\exists r_{\min}>0$ such that all stationary points are $(r_{\min},0)$-stable).
\vspace{-0.3cm}
\paragraph{(A1)}
$\forall t>0$, we assume there is an upper bound $B$ on 
\begin{eqnarray*}
\max
\{
E[\sum_r\sum_{x\in A_r^{t+1}}\|x-\hat{c}_r^{t+1}\|^2+\phi^t|F_t],
\sum_{r}n_r^*\langle c_r^{t-1}-c_r^*,\hat{c}_r^t - E[\hat{c}_r^t|F_{t-1}]\rangle, 
\sum_{r}n_r^*\|\hat{c}_r^{t}-c_r^*\|^2 
\}
\end{eqnarray*}
where $n_r^*, c_r^*$ are the cardinality and centroid of the $r$-th cluster in a stationary clustering $A^*$.
\begin{thm}\label{thm:km}
Assume (A0)(A1) holds. Fix any $\delta >0$, if we run Algorithm \ref{alg:MBKM} with learning rate $\eta^t=\frac{c^{\prime}}{t+t_o}$ 
such that for all $t\ge 1$, 
$$
c^{\prime}> \frac{1}{2 p^{+}_{\min}(m)(1-\sqrt{1-\frac{r_{\min}\phi^{opt}}{2\phi^t}})}
\mbox{\quad and\quad}
t_o= 
\Omega
\left\{[\frac{(c^{\prime})^2}{\rho(m)^2(2c^{\prime}-1)}\ln\frac{1}{\delta}]^{\frac{2}{2c^{\prime}-1}}
\right\}
$$ 
$$
\mbox{
where~
}
p^{+}_{\min}(m):=\min_{\substack{r\in [k];\\p_r^t(m)>0}}p_r^t(m),
\rho(m):=1 - (1-p_{\min}^{*})^m, 
\mbox{~and~} 
p^*_{\min}:=\min_{A^{**}\in \{A^*\}_{[k]}}\min_r\frac{n^{**}_r}{n}
$$
Then
starting from any initial set of $k$-centroids $C^0$, 
Let $G$ denote the event 
$\{\exists T, \exists A^{**}\in \{A^*\}_{[k]} \mbox{~s.t.~} A^t = A^{**}, \forall t\ge T \}$.
Then 
$
Pr(G)\ge 1-\delta
$,
and there exists events parametrized by $A^{**}$, denoted by
$G_o(A^{**})$, such that 
$
Pr\{\cup_{A^{**}\in \{A^*\}_{[k]}} G_o(A^{**})\}\ge 1-\delta
$. And for any 
$G_o(A^{**})$, we have
$\forall t\ge 1$,
$
E\{\phi^t-\phi^{**}|G_o(A^{**})\} = O(\frac{1}{t}) \mbox{,~~where~~} \phi^{**}:=\phi(m(A^{**})).
$
\end{thm}
This result guarantees the expected $O(\frac{1}{t})$ convergence of Algorithm \ref{alg:MBKM}
towards a stationary point under a general condition.
However, it does not guarantee Algorithm \ref{alg:MBKM} converges to the same stationary clustering as its batch counterpart, even with the same initialization. Moreover, it is even possible that the final solution becomes degenerate. 
\paragraph{Algorithm parameters in the convergence bound}
The exact convergence bound of Theorem \ref{thm:km}, which we hide in the Big-\textit{O} notation, reveals dependence of the convergence on the algorithm's parameters.
When $c^{\prime}$ is sufficiently large, the exact bound we obtain is likely dominated by 
$
\frac{a_{\max}^2B}{\beta-1}
(\frac{t_0+2}{t_0+1})^{\beta+1}\frac{1}{t_0+t+1}
$ 
where $\beta=2c^{\prime}p_{\min}^{+}(m)(1-\sqrt{\frac{\tilde\phi^t}{\phi^t}})$
and 
$a_{\max}:=\frac{c^{\prime}}{\rho(m)}$.
The bound becomes tighter when $c^{\prime}$ decreases.
Since the $\frac{1}{t}$-order convergence requires a sufficiently large $c^{\prime}$, our analysis suggests we should choose $c^{\prime}$ to be neither too large nor too small.
The bound also becomes tighter when $p_{\min}^{+}(m)$ and $\rho(m)$ becomes closer to $1$. Both depend on $m$; less obviously, they also depend on the number of non-degenerate clusters.
Since $p_r^t=1-(1-\frac{n_r^t}{n})^m$, which equals zero if and only if $n_r^t=0$,
$p_{\min}^{+}(m)=1-(1-\min_{r,t;p_r^t>0}p_r^t(1))^m
<1-(1-\frac{1}{k^{\prime}})^m
$, where $k^{\prime}$ is the smallest number of non-degenerate centroids in a run. The similar holds for $\rho(m)$.
This suggests the convergence rate may be slower with larger $k^{\prime}$, which depends on the initial number of clusters $k$, and smaller $m$. For experimental results of the effect of algorithm parameters on convergence, see Section \ref{sec:EXP}.
A detailed explanation on the exact bound is included in Appendix B.
\begin{algorithm}[t]
\caption{Buckshot seeding \cite{tang_montel:aistats16}}
\label{alg:seeding}
\begin{algorithmic}
\STATE $\{\nu_i,i\in [m_o]\} \leftarrow $ sample $m_o$ points from $X$ uniformly at random with replacement
\STATE $\{S_1,\dots,S_k\} \leftarrow $run Single-Linkage on $\{\nu_i,i\in [m_o]\}$ until there are only $k$ connected components left
\STATE $C^0=\{\nu_r^{*},r\in [k]\}\leftarrow$ take the mean of the points in each connected component $S_r,r\in [k]$
\end{algorithmic}
\end{algorithm}
The general convergence result in Theorem \ref{thm:km} is applicable to any seeding algorithm, including random sampling, which is probably the most scalable seeding heuristic one can use. However, it does not provide performance guarantee with respect to the $k$-means objective. We next show that stochastic $k$-means initialized by Algorithm \ref{alg:seeding} converges to a global optimum of $k$-means objective at rate $O(\frac{1}{t})$, under additional geometric assumptions on the dataset. 
The major advantage of Algorithm \ref{alg:seeding} over other choices of initialization, such as the \texttt{k-means++} \cite{kmeanspp} or random sampling, is that its running time is independent of the data size while providing seeding guarantee. When using it with stochastic $k$-means, both the seeding and the update phases are independent of the data size, making the overall algorithm highly scalable.

Let $(A^{opt},C^{opt})$ denote an optimal $k$-means clustering of $X$. We assume, similar to \cite{kumar}, that the means of each pair of clusters are well-separated and that the points from the two clusters are separated by a ``margin'', that is, $\forall x\in A_r^{opt}\cup A_s^{opt}$, the distance from $x$ to the bisector of $c_r^{opt}$ and $c_s^{opt}$ is lower bounded.
Formally, let $\bar{x}$ denote the projection of $x$ onto the line joining $c_r^{opt},c_s^{opt}$, the margin between the two clusters is defined as
$
\Delta_{rs}:=\min_{x\in A_r^{opt}\cup A_s^{opt}} |\|\bar{x}-c_r^{opt}\| - \|\bar{x}-c_s^{opt}\||
$.
In addition, we require that the size of the smallest cluster is not too small compared to the data size.
The geometric assumptions are formally defined as below.
\paragraph{(B1)}
Mean separation: $\forall r\in [k], s\ne r$, $\|m(A^{opt}_r)-m(A^{opt}_s)\|\ge f(\alpha)\sqrt{\phi^{opt}}(\frac{1}{\sqrt{n_r^{opt}}}+\frac{1}{\sqrt{n_s^{opt}}})$, with $f(\alpha)>\max\{64^2,\frac{5\alpha+5}{256\alpha},\max_{r\in [k],s\ne r} \frac{n_r^{opt}}{n_s^{opt}}\}$ for some $\alpha \in (0,1)$ that is sufficiently small.
\vspace{-.3cm}
\paragraph{(B2)}
Existence of margin: $\forall r\in [k], s\ne r$, 
$\Delta_{rs}\ge \gamma\|m(A^{opt}_r)-m(A^{opt}_s)\|$, for some $\gamma>\frac{8\sqrt{2}}{\sqrt{f}}$.
\vspace{-.3cm}
\paragraph{(B3)}
Cluster balance: $p_{\min}\ge \frac{\gamma}{16^2f(\alpha)}+\sqrt{\alpha}$, where $p_{\min}:=\min_{r\in [k]}\frac{n_r^{opt}}{n}$.
\begin{thm}
\label{thm:solution}
Assume (B1)(B2) (B3) hold for an optimal $k$-means clustering $A^{opt}$.
Fix any $\xi > 0$, if $f(\alpha)$ in addition satisfies
$
f(\alpha)\ge 5\sqrt{\frac{1}{2w_{\min}}\ln (\frac{2}{\xi p_{\min}}\ln\frac{2k}{\xi})}
$, 
and we choose $m_o$ in Algorithm \ref{alg:seeding} such that
$
\frac{\log\frac{2k}{\xi}}{p_{\min}}
<m_o
<\frac{\xi}{2}\exp\{2(\frac{f(\alpha)}{4}-1)^2w_{\min}^2\}
$.
And if we run Algorithm \ref{alg:MBKM} initialized by Algorithm \ref{alg:seeding} and choosemini-batch size $m$ and learning rate of the form 
$\eta^t = \frac{c^{\prime}}{t+t_o}$ so that
$$
m>1\mbox{~~~and~~~}
c^{\prime}>\frac{1}{2[1-\sqrt{\alpha}-(1-\sqrt{\alpha})^m]}$$
and
$$
t_o= 
\Omega
\left\{[\frac{(c^{\prime})^2}{\rho(m)^2(2c^{\prime}(\rho(m)-\sqrt{\alpha})-1)}\ln\frac{1}{\delta}]^{\frac{2}{2c^{\prime}(\rho(m)-\sqrt{\alpha})-1}}
\right\}
$$
where
$
\rho(m):=1 - [1-(p_{\min} - \frac{\gamma}{16^2 f(\alpha)})]^m
$.
Then $\forall t\ge 1$, there exists event $G_t\subset \Omega$ s.t.
$$
Pr\{G_t\}\ge (1-\delta)(1-\xi)
\mbox{~~~and~~~}
E[\phi^t|G_t]-\phi^{opt}\le E[\Delta(C^t,C^{opt})|G_t] = O(\frac{1}{t})
$$
\end{thm}
Interestingly, we cannot provide performance guarantee for online $k$-means ($m=1$) in Theorem \ref{thm:solution}.
Our intuition is, instead of allowing stochastic $k$-means to converge to any stationary point as in Theorem \ref{thm:km}, it studies convergence to a fixed optimal solution; a larger $m$ provides more stability to the algorithm and prevents it from straying away from the target solution.
The proposed clustering scheme is reminiscent to the \textit{Buckshot algorithm} \cite{buckshot}, widely used in the domain of document clustering.
Readers may wonder how can the algorithm approach $\phi^{opt}$ for an NP-hard problem. The reason is that our geometric assumptions softens the problem.
In this case, Algorithm \ref{alg:MBKM} converges to the same (optimal) solution as its batch counterpart, provided the same initialization.
\subsection{Related work and proof outline}\label{sec:related_work}
Our major source of inspiration comes from recent advances in non-convex stochastic optimization for unsupervised learning problems \cite{ge:sgd_tensor, balsubramani13}.
\cite{ge:sgd_tensor} studies the convergence of stochastic gradient descent (SGD) for 
the tensor decomposition problem, which amounts to finding a local optimum of a non-convex objective function composed exclusively of saddle points and local optima. 
Inspired by their analysis framework, we divide our analysis of Algorithm \ref{alg:MBKM} into two phases, that of global convergence and local convergence, indicated by the distance from the current solution to stationary points, $\Delta(C^t,C^*)$.
We use $\Delta^t:=\Delta(C^t,C^{*})$ as a shorthand.
\paragraph{Significant decrease in $k$-means objective when $\Delta^t$ is large}
In the global convergence phase, the algorithm is not close to any stationary point, i.e.,
$\forall C^*\in\{C^*\}_{[k]}$, $\Delta(C^t,C^*)$ is lower bounded.
We first prove a black-box result showing that on non-stationary points, stochastic $k$-means decreases the $k$-means objective in expectation at every iteration.
\begin{lm}\label{lm:kmeans}
Suppose $\forall r\in [k]$, $\eta_r^t\le \eta_{\max}^t<1$ w.p. $1$. 
Then,
$
E[\phi^{t+1}-\phi^t|F_t]
\le
-2\min_{r,t; p_r^{t+1}(m)>0}[\eta_{r}^{t+1}p_r^{t+1}(m)]\phi^t
(1-\sqrt{\frac{\tilde\phi^{t}}{\phi^t}})
+(\eta_{\max}^{t+1})^2E[\sum_r\sum_{x\in A_r^{t+1}}\|x-\hat{c}_r^{t+1}\|^2+\phi^t|F_t]
$,
where $\tilde{\phi}^t:=\sum_r\sum_{x\in A_r^{t+1}}\|x-m(A_r^{t+1})\|^2$.
\end{lm}
By Lemma \ref{lm:kmeans}, the term $\min_{r,t; p_r^{t+1}(m)>0}[\eta_{r}^{t+1}p_r^{t+1}(m)]\phi^t
(1-\sqrt{\frac{\tilde\phi^{t}}{\phi^t}})$
lower bounds the per iteration drop in $k$-means objective. 
Since $\min_{r;p_r^{t+1}(m)>0}p_r^{t+1}(m)\ge \frac{1}{n}$, $\phi^t\ge \phi^{opt}$,
and $1-\frac{\tilde\phi_{t}}{\phi^t}= \frac{\phi^t-\tilde\phi_{t}}{\phi^t}=\frac{\Delta(C^t, C^{t+1})}{\phi^t}\ge 0$ (the second equality is by Lemma \ref{centroidal}), the drop is always lower bounded by zero.
We show that if $\Delta(C^t,C^*)$ is bounded away from zero (for any stationary point $C^*$), then so is $\Delta(C^t, C^{t+1})$: the rough idea is, in case $C^{t+1}$ is a non-stationary point, $C^t$ and $C^{t+1}$ must belong to different equivalence classes, as discussed in Section \ref{sec:framework}, and their distance must be lower bounded by Lemma \ref{lm:closure}; otherwise, $C^{t+1}$ is a stationary point, and by our assumption their distance is lower bounded. 
Thus, 
$\frac{\Delta(C^t, C^{t+1})}{\phi^t}$ is lower bounded by a positive constant, and so is $1-\sqrt{\frac{\tilde\phi^{t}}{\phi^t}}$. 
Since we chose $\eta^t:=\Theta(\frac{1}{t})$, the expected per iteration drop of cost is of order $\Omega(\frac{1}{t})$, which forms a divergent series; after a sufficient number of iterations the expected drop can be arbitrarily large. We conclude that $\Delta(C^t,C^*)$ cannot be bounded away from zero asymptotically, since the $k$-means cost of any clustering is positive. This result is presented in Lemma \ref{lm:limit_cluster}.

Before proceeding, note
the drop increases with 
$\min_{r,t; p_r^{t+1}(m)>0}[\eta_{r}^{t+1}p_r^{t+1}(m)]$.
This means if we can set a cluster-dependent learning rate that adapts to $p_r^{t+1}(m)$, the drop could be larger than choosing a flat rate, as we do in our analysis.
The learning rate in \cite{Sculley,BottouBengio}, where $\eta^t_r:=\frac{\hat{n}_r^t}{\sum_{i\le t}\hat{n}_r^i}$, is intuitively making this adaptation:
in the case when the clustering assignment does not change between iterations, it can be seen that
$E\eta_r^t\approx \frac{1}{t p_r^t(m)}$, so the effective learning rate $\eta_r^t p_r^t(m)$ is balanced for different clusters and is roughly $\frac{1}{t}$. However, in practice, the clustering assignment changes when the centroids are updated, and it is hard to analyze this adaptive learning rate due to its random dependence on all previous sampling outcomes.
 \begin{lm}
 \label{lm:limit_cluster}
Assume (A1) holds.
 If we run Algorithm \ref{alg:MBKM} on $X$ with $\eta^t=\frac{a}{t+t_o}, a\ge\frac{1}{2p_{\min}^{+}(m)}$, with $p_{\min}^{+}(m)$ as defined
 in Theorem \ref{thm:km},
 and $t_o>1$, and choose any initial set of $k$ centroids $C^0$.
Then for any $\delta>0$, $\exists t$ s.t. $\Delta(C^t,C^*)\le \delta$ with $C^*:=m(A^*)$ for some $A^*\in \{A^*\}_{[k]}$.
\end{lm}
Lemma \ref{lm:limit_cluster} suggests that, starting from any initial point $C^0$ in $\{C\}$, the algorithm always approaches a stationary point asymptotically, ending its global convergence phase after a finite number of iterations. We next examine its local behavior around stationary points.
\paragraph{Local convergence to stationary points when $\Delta^t$ is small}
To obtain local convergence result, we first define ``local attractors'' and ``basin of attraction'' for batch $k$-means; the natural candidate for the former is the set of stationary points; a basin of attraction is a subset of the solution space so that once batch $k$-means enters it, it will not escape. 
\begin{defn}\label{defn:stable}
We call $C^{*}$ a $(b_0,\alpha)$-stable stationary point of batch $k$-means if it is a stationary point and for any clustering $C$ such that 
$\Delta(C,C^*)\le b^{\prime}\phi^*$, $b^{\prime}\le b_0$, 
the Voronoi partition induced by $\{c_r\}$, denoted by $\{A_r\}$, satisfies
$\max_r\frac{|A_{\pi(r)}\triangle A_r^*|}{n^*_r}\le \frac{b}{5b+4(1+\phi(C)\slash\phi^*)}$, where
$\pi$ is the permutation achieving $\Delta(C,C^*)$,
with $b\le \alpha b^{\prime}$ for some $\alpha\in [0,1)$.
\end{defn}
We can view $b_0$ as the radius of the basin and $\alpha$ the degree of the ``slope'' leading to the attractor.
A key lemma to our analysis characterizes the iteration-wise local convergence property of batch $k$-means around stable stationary points, whose convergence rate depends on the slope $\alpha$. 
\begin{lm}
\label{lm:stable_stat}
Let $C^*$ be a $(b_0,\alpha)$-stable stationary point. For any $C$ such that 
$\Delta(C,C^{*})\le b^{\prime}\phi^*$, $b^{\prime}\le b_0$,
apply one step of batch $k$-means update on $C$ results in a new solution $C^1$ such that
$\Delta(C^1,C^*)\le \alpha b^{\prime}\phi^*$.
\end{lm}
Lemma \ref{lm:stable_stat} resembles the standard iteration-wise convergence statement in SGD analysis, typically via convexity or smoothness of a function \cite{rakhlin}. Here, we have neither at our dispense (we do not even have a gradient). Instead, our analysis relies on the geometric property of Voronoi diagram and the mean operation used in a $k$-means iteration, similar to those in recent works on batch $k$-means \cite{kumar,awasthi:improved,tang_montel:aistats16}. 
Although this lemma applies only to batch $k$-means, our hope is that stochastic $k$-means has similar iteration-wise convergence behavior in expectation even in the presence of noise. 

The difficulty here is, due to non-convexity, the convergence result only holds within a local basin of attraction: if the algorithm's solution is driven off the current basin of attraction by stochastic noise at any iteration, it may converge to a different attractor, causing trouble to our analysis.
To deal with this, we exploit probability tools developed in \cite{balsubramani13}. \cite{balsubramani13} studies the convergence of stochastic PCA algorithms, where the objective function is the non-convex Rayleigh quotient, which has a plateau-like component. The tools developed there were used to show that stochastic PCA gradually escapes the plateau. Here, we adapted their analysis to show Algorithm \ref{alg:MBKM} stays within a basin of attraction with high probability, and converges to the attractor at rate $O(\frac{1}{t})$.
We define a nested sequence of events $\Omega_{i}\subset \Omega$:
$$
\Omega \supset \Omega_{1} \supset\dots \supset\Omega_i\supset \dots
$$
where $\Omega_i:=\{\Delta^{t}\le b_0\phi^{*},\forall t<i\}$.
Then if Algorithm \ref{alg:MBKM} is within the basin of attraction of a stable stationary point at time $t$, the event that it escapes this local basin of attraction is contained in the 
event $\cup_{t\ge i+1}\Omega_{t-1}\setminus\Omega_{t}$.
We upper bound the probability of this bad event (Proposition \ref{prop:high_prob}) using techniques that derive tight concentration inequality via moment generating functions from \cite{balsubramani13}, which in turn implies a lower bound on the probability of $\Omega_t$, $t\ge i$.
%
Then conditioning on $\Omega_t$ and adapting Lemma \ref{lm:stable_stat} from batch to stochastic $k$-means proves the expected local convergence rate of $O(\frac{1}{t})$.
\begin{thm}
\label{thm:mbkm_local}
Assume (A1) holds.
Let $C^*$ be a $(b_0,\alpha)$-stable stationary point,
and let $\Delta^i = b\phi^*$ for some $b\le \frac{1}{2}b_0$ at some iteration $i$
in Algorithm \ref{alg:MBKM}.
Let  
$
a^t:=\frac{\max_{r}p_r^t(m)}{\min_s p_s^t(m)}
$. 
Suppose we set $c^{\prime}$ and $m$ sufficiently large so that 
$$
\beta:=2c^{\prime}\min_{r,t}p_r^t(m)(1-\max_{t}a^t\sqrt{\alpha}))>1
$$
Fix any $0<\delta\le \frac{1}{e}$, and let $a_{\max}:=\frac{c^{\prime}}{\min_{r,t}p_r^t(m)}$.
If in addition,
$$
t_0\ge \max\{
(\frac{16a_{\max}^2B}{\Delta^i})^{\frac{1}{\beta-1}},
[\frac{48a_{\max}^2B^2}{(\beta-1)(\Delta^i)^2}\ln\frac{1}{\delta}]^{\frac{2}{\beta-1}},
(\ln\frac{1}{\delta})^{\frac{2}{\beta-1}}
\}
$$
Then for all $t> i$, 
$$
Pr(\Omega_t)\ge 1-\delta 
\mbox{~~and~~}
E[\Delta^t|\Omega_t] \le 
\frac{t_0+i+1}{t_0+t}\Delta^i
+\frac{a_{\max}^2B}{\beta-1}(\frac{t_0+i+2}{t_0+i+1})^{\beta+1}\frac{1}{t_0+t+1}
$$
\end{thm}
Note how, in contrast to Theorem \ref{thm:km},
the local convergence result does not allow degeneracy here, by implicitly requiring that $\min_{r,t}p_r^t(m)>0$.
This is reasonable, since a degenerate set of $k$ centroids cannot converge to a fixed stationary point
with $k$ centroids.

Building on the ideas introduced above, we present the key ingredients of our main theorems. 
\vspace{-0.2cm}
\paragraph{Proof idea of Theorem \ref{thm:km}}
To prove Theorem \ref{thm:km}, we first show that all stationary points satisfying (A0) must be
$(r_{\min},0)$-stable for some $r_{\min}>0$ (Lemma \ref{lm:stat_stab}).
Then we apply results from both the global and local convergence analysis of Algorithm \ref{alg:MBKM}.
We define the global convergence phase as when $\Delta(C^t,C^{**})>\frac{1}{2}r_{\min}\phi^*$, $\forall C^{**}\in \{C^*\}_{[k]}$. As discussed, since $\Delta(C^t,C^{**})$ is bounded away from zero, the per-iteration drop of $k$-means cost is of order $\Omega(\eta^t)$, thus we get the expected $O(\frac{1}{t})$ convergence rate.
Lemma \ref{lm:limit_cluster} suggests that this phase will eventually end and at some iteration $T$,
$\Delta(C^T,C^{**})\le \frac{1}{2}r_{\min}\phi^{**}$ must hold for some stationary point $C^{**}$.
The first time when this happens signals the beginning of the local convergence phase:
at this point, $C^T$ is in the basin of attraction of $C^{**}$ since $C^{**}$ is $(r_{\min},0)$-stable.
Thus, applying Theorem \ref{thm:mbkm_local} implies that in this phase the algorithm converges to $C^{**}$ locally at rate $O(\frac{1}{t})$.
Hence, the overall global convergence rate is $O(\frac{1}{t})$.
\vspace{-0.2cm}
\paragraph{Proof idea of Theorem \ref{thm:solution}}
Here we only apply the local convergence result. 
The key step is to show that our clusterability assumption on the dataset, \textbf{(B1)} and \textbf{(B2)}, implies that its optimal $k$-means solution, $C^{opt}$, is a $(b_0,\alpha)$-stable stationary point with a sufficiently large $b_0$ (Proposition \ref{prop:geom}). Then we adapt results from \cite{tang_montel:aistats16} to show that the seeds obtained from Algorithm \ref{alg:seeding} are within the basin of attraction of $C^{opt}$ with high probability (Lemmas \ref{lm:adapt_tang}). Using the other geometric assumption, \textbf{(B3)}, we apply Theorem \ref{thm:mbkm_local} to show an $O(\frac{1}{t})$ convergence to $C^{opt}$ with high probability.
\begin{figure}
\begin{subfigure}{.55\textwidth}
  \centering
  \includegraphics[width=\linewidth, height = 0.42\textwidth]{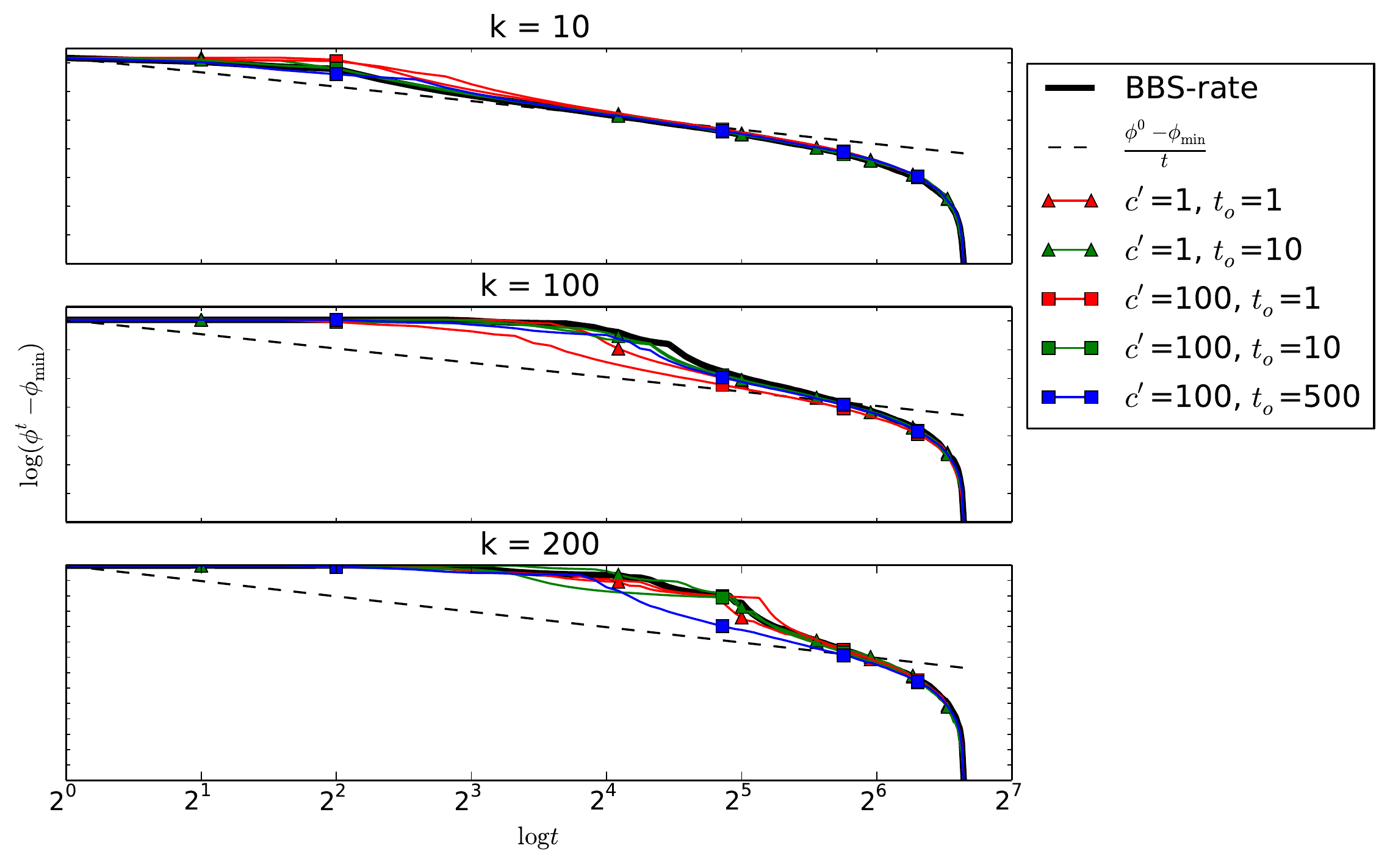}
  \caption{$m=1$, log-log plot}
  \label{fig:mb1_log}
\end{subfigure}%
\begin{subfigure}{.4\textwidth}
  \centering
  \includegraphics[width=\linewidth, height = 0.54\textwidth]{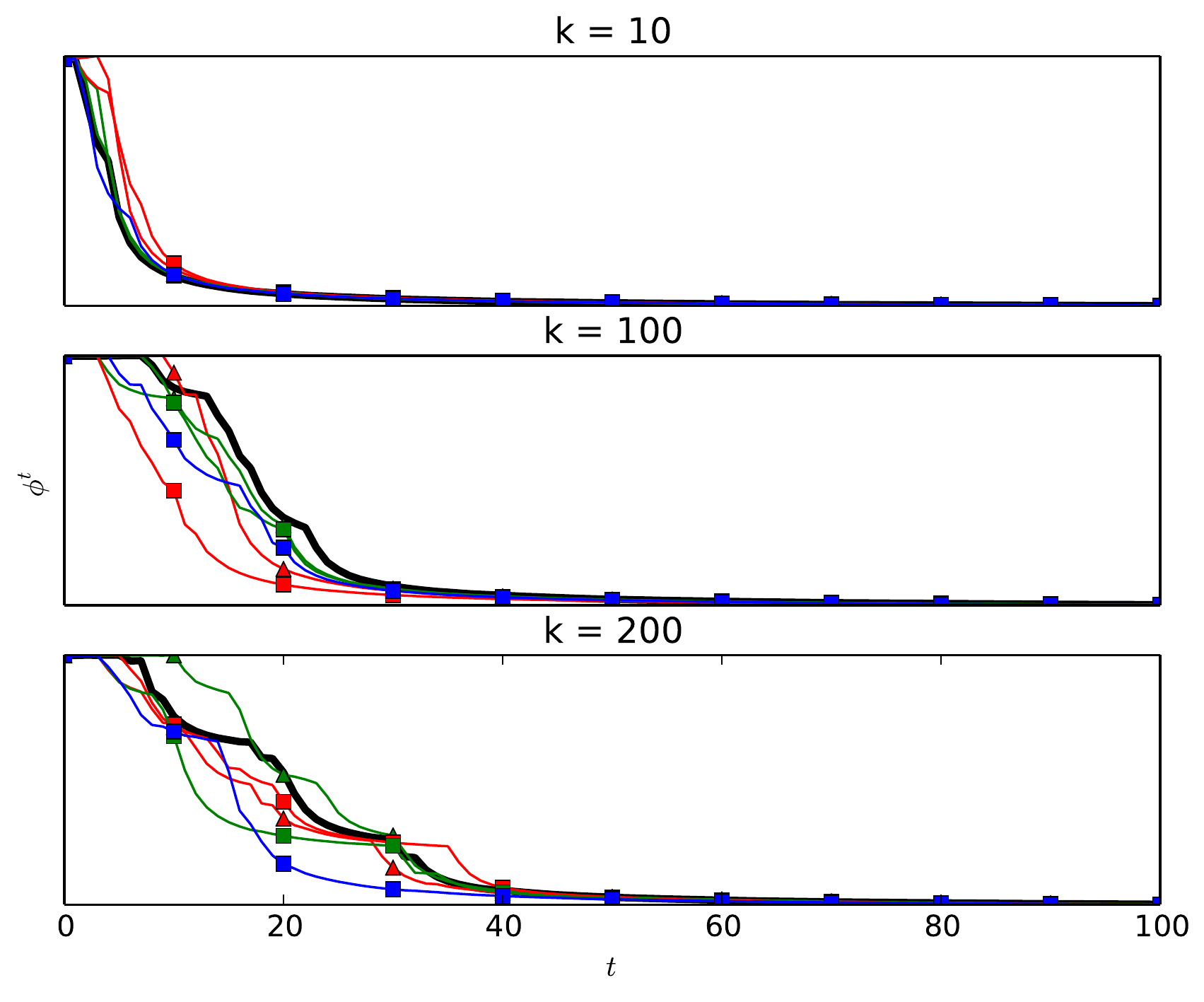}
  \caption{$m=1$, original plot}
  \label{fig:mb1_origin}
\end{subfigure}
\vspace{-0.1cm}
\vspace{-0.3cm}
\end{figure}
\setcounter{figure}{1}
\begin{figure}
\begin{subfigure}{.33\textwidth}
  \centering
  \includegraphics[width=\linewidth, height = 0.6\textwidth]{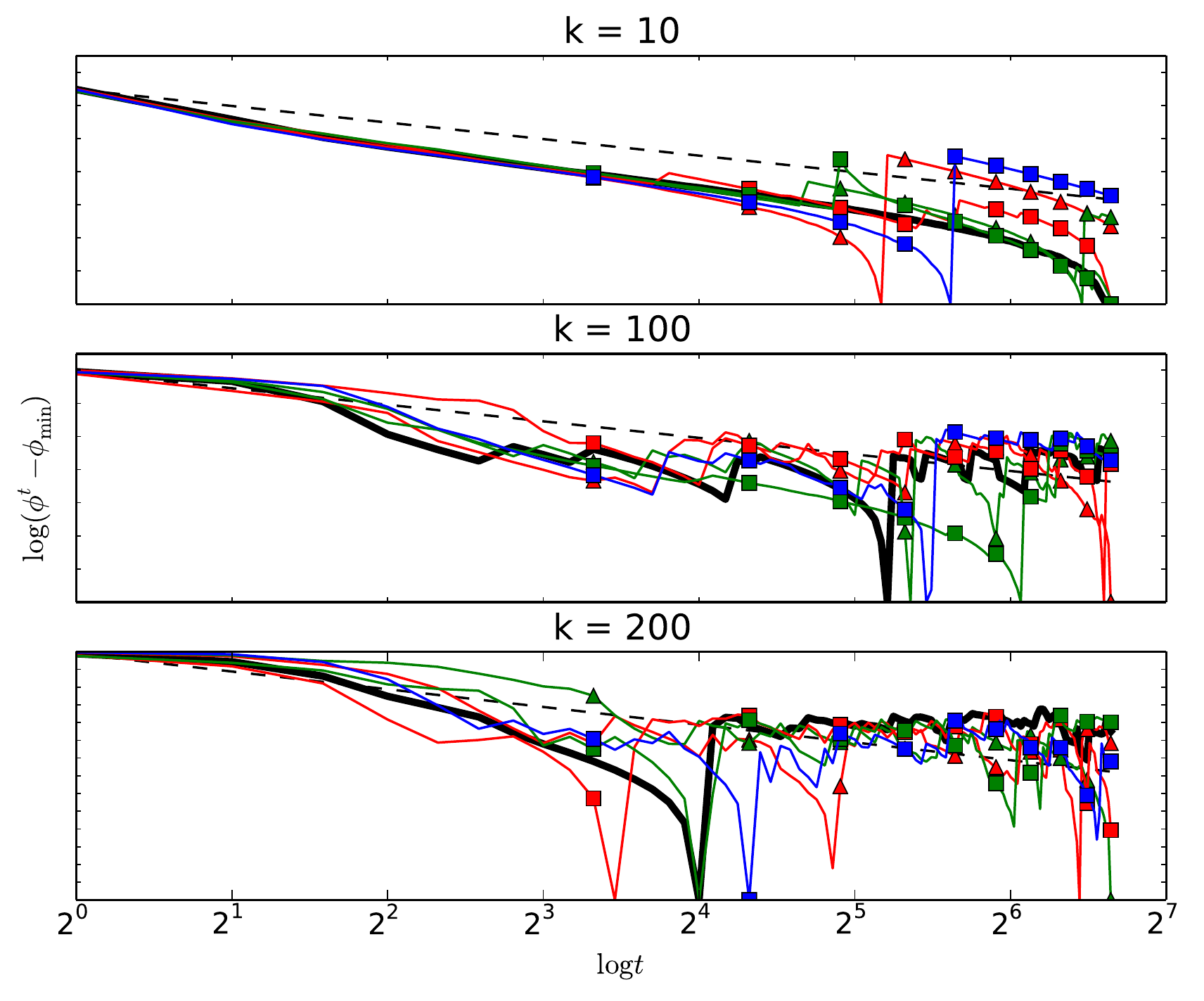}
  \caption{$m=10$}
  \label{fig:mb10_log}
\end{subfigure}%
\begin{subfigure}{.33\textwidth}
  \centering
  \includegraphics[width=\linewidth, height = 0.6\textwidth]{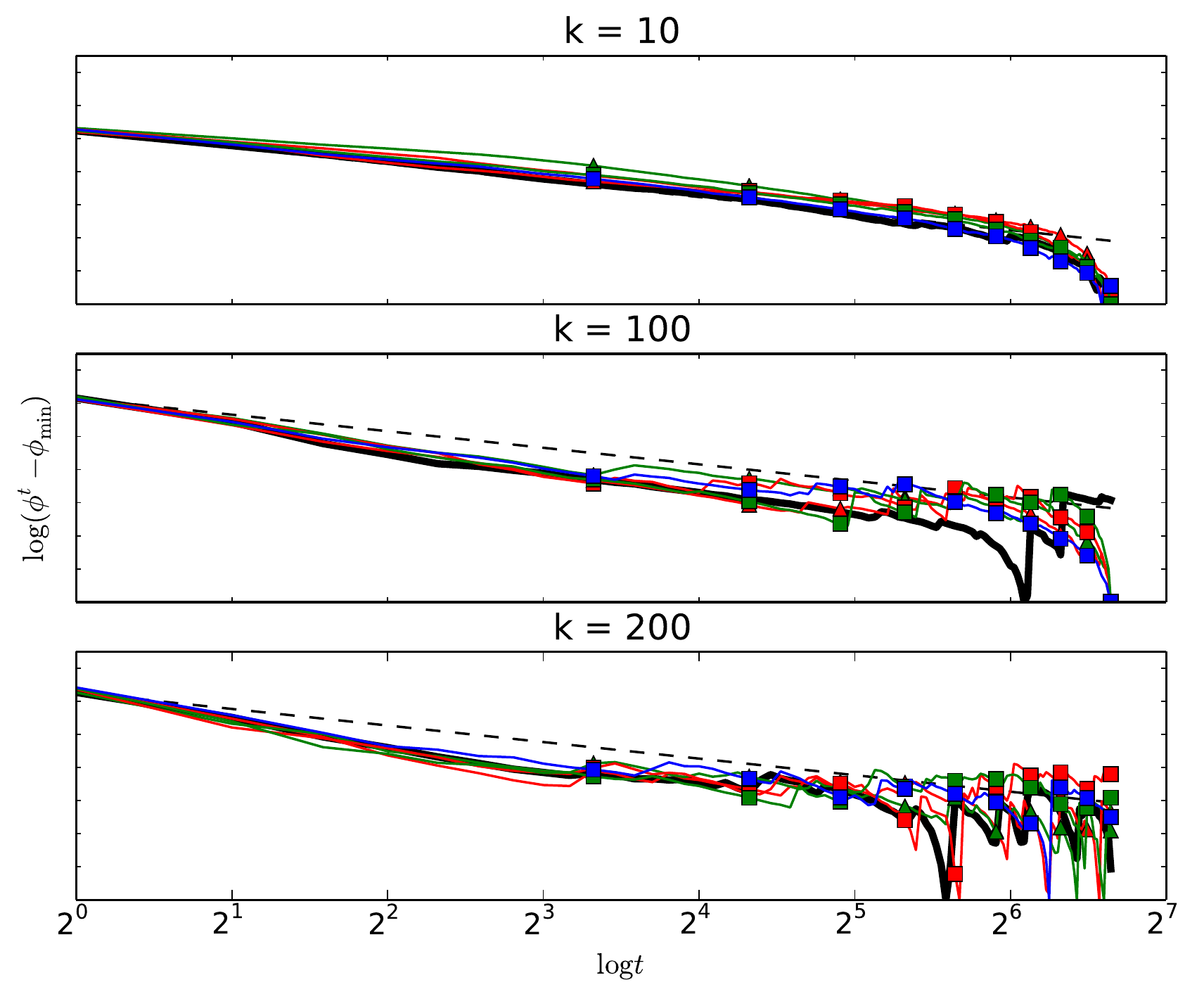}
  \caption{$m=100$}
  \label{fig:mb100_log}
\end{subfigure}
\begin{subfigure}{.33\textwidth}
  \centering
  \includegraphics[width=\linewidth, height = 0.6\textwidth]{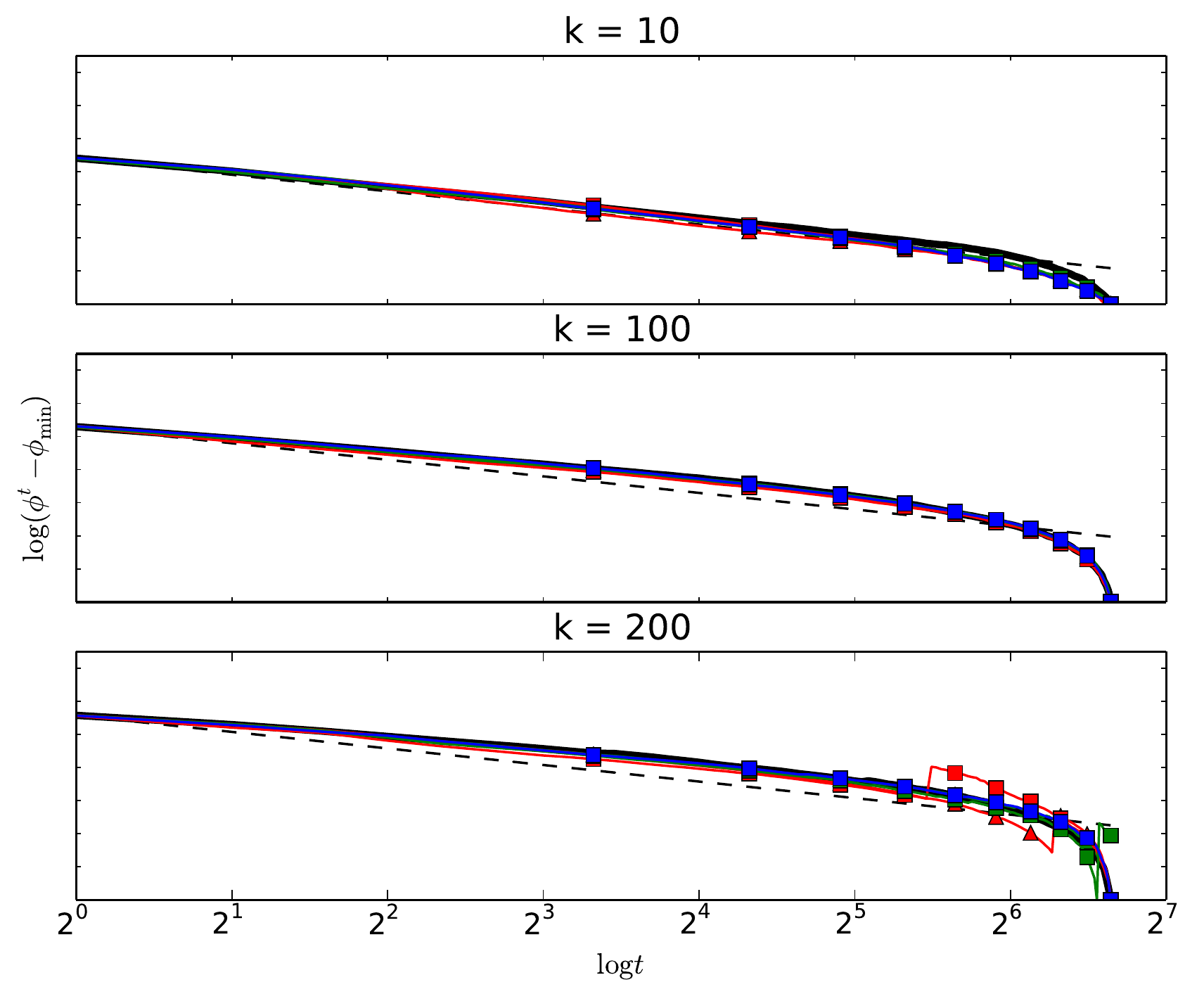}
  \caption{$m=1000$}
  \label{fig:mb1000_log}
\end{subfigure}
\caption{Convergence graphs of mini-batch $k$-means}
\label{fig:log}
\vspace{-0.4cm}
\end{figure}
\section{Experiments}\label{sec:EXP}
To verify the $O(\frac{1}{t})$ global convergence rate of Theorem \ref{thm:km}, we run stochastic $k$-means with varying learning rate, mini-batch size, and $k$ on \texttt{RCV1} \cite{data:rcv1}. The dataset is relatively large in size: it has manually categorized $804414$ newswire stories with $103$ topics, where each story is a $47236$-dimensional sparse vector; it was used in \cite{Sculley} for empirical evaluation of mini-batch $k$-means.
We used \texttt{Python} and its \texttt{scikit-learn} package \cite{scikit-learn} for our experiments, which has stochastic $k$-means implemented. We disabled centroid relocation and modified their source code to allow a user-defined learning rate (their learning rate is fixed as $\eta^t_r:=\frac{\hat{n}_r^t}{\sum_{i\le t}\hat{n}_r^i}$, as in \cite{BottouBengio, Sculley}, which we refer to as \textbf{BBS-rate}). 

Figure \ref{fig:log} shows the convergence in $k$-means cost of stochastic $k$-means algorithm over $100$ iterations for varying $m$ and $k$; fix each pair $(m, k)$, we initialize Algorithm \ref{alg:MBKM} with a same set of $k$ random centroids and run stochastic $k$-means with varying learning rate parameters $(c^{\prime}, t_o)$, and we average the performance of each learning rate setup over $5$ runs to obtain the original convergence plot. 
Figure \ref{fig:mb1_origin} is an example of a convergence plot before transformation. The dashed black line in each log-log figure is $\frac{\phi^0-\phi_{\min}}{t}$, a function of order $\Theta(\frac{1}{t})$. To compare the performance of stochastic $k$-means with this baseline, we first transform the original $\phi^t$ vs $t$ plot to that of $\phi^t-\phi_{\min}$ vs $t$. By Theorem \ref{thm:km}, $E[\phi^t-\phi^{**}|G(A^{**})]=O(\frac{1}{t})$, so we expect the slope of the log-log plot of $\phi^t-\phi^{**}$ vs $t$ to be at least as large as that of $\Theta(\frac{1}{t})$. Although we do not know the exact cost of the stationary point, since the algorithm has reached a stable phase over $100$ iterations, as illustrated by Figure \ref{fig:mb1_origin}, we simply use $\phi_{\min}$ as an estimate of $\phi^{**}$. 
Most log-log convergence graphs fluctuate around a line with a slope at least as steep as that of $\Theta(\frac{1}{t})$, and do not seem to be sensitive to the choice of learning rate in our experiment. Note in some plots, such as Figure \ref{fig:mb1_log}, the initial phase is flat. This is because we force the plot to start at $\phi^0-\phi_{\min}$ instead of their true intercept on the $y$-axis. \textbf{BBS-rate} exhibits similar behavior to our flat learning rates.
Our experiment suggests the convergence rate of stochastic $k$-means may be sensitive to the ratio $\frac{m}{k}$; for larger $m$ or smaller $k$, faster and more stable convergence is observed. 
\small{
\bibliography{mbkm_nips16}
\bibliographystyle{plain}
}
\newpage
\section{Appendix A: complete version of Section \ref{sec:framework}}
To facilitate our analysis of mini-batch $k$-means, we build a framework to
track the solution path produced by batch $k$-means;
it alternates between two solutions spaces: the space of all $k$-centroids, which we denote by $\{C\}$, and the space of all $k$-clusterings, which we denote by $\{A\}$. 
Note the latter is a finite set. Throughout the paper, we use $\pi: [k]\rightarrow [k]$ to denote permutations between two sets of the same cardinality. 
\paragraph{Degenerate cases}
One problem with batch $k$-means algorithm is it may produce degenerate solutions: if the solution $C^t$ has $k$ centroids, it is possible that data points are mapped to only $k^{\prime}<k$ centroids. To handle degenerate cases, starting with $|C^0|=k$, we will
consider the enlarged clustering space $\{A\}_{[k]}$, the union of all $k^{\prime}$-clusterings, which we denote by $\{A\}_{k^{\prime}}$,
with $1\le k^{\prime}\le k$.
\paragraph{Partitioning $\{C\}$ via $\{A\}$}
Our first observation is that most part of $\{C\}$ (with exception discussed below) can be partitioned into equivalence classes by the clustering they induce on $X$;
for any $C$, let 
$
v(C):= V(C)\cap X \in \{A\}_{[k]}
$ to formally denote the clustering they induce via their Voronoi diagram.
For most points in $\{C\}$, there is only one such clustering so $v(C)$ is uniquely determined. We define $v^{-1}(A)$ as the set of points inducing a unique $k^{\prime}$-clustering $A$, with $k^{\prime}\le k$.
Then we let $C_1,C_2$ be in the same partition in $\{C\}$ if $v(C_1),v(C_2)$ are both unique and $v(C_1)=v(C_2)$.
\paragraph{Ambiguous cases}
However, it is not always clear which partition $C$ belongs to: if 
$\exists x\in X$ such that $\|x-c_1(x)\|=\|x-c_2(x)\|$, where $c_1(x), c_2(x)$ denote the centroids in $C$ that are closest and second closest to $x$, $x$ can be clustered into either the cluster of centroid $c_1(x)$ or that of $c_2(x)$. Centroids with this property can induce two or more clusterings due to ambiguity of Voronoi partition.
Intuitively, they are at the boundary of $v^{-1}(A^i)$, for some clusterings $A^i$. 
We formalize $v^{-1}(A)$ and boundary points as below.
\begin{defn}[members of $v^{-1}(A)$]
\label{defn:members}
Fix a clustering $A=\{A_1,\dots,A_k\}$, we define
$C\in v^{-1}(A)$ if it contains a matching set of centroids and there exists a permutation of $[k]$
s.t.
$\forall x\in A_r, \forall r\in [k]$, 
$$
\|x-c_{\pi(r)}\|<\|x-c_{\pi(s)}\|, \forall s\ne r
$$
We say $C$ is a \textbf{boundary point} of $v^{-1}(A)$, if
$\forall r\in [k], \forall x\in A_r$, 
$$
\|x-c_{\pi(r)}\|\le\|x-c_{\pi(s)}\|, \forall s\ne r
$$ with equality attained for at least one data point $x$.
Let $B(v^{-1}(A))$ denote the set of all boundary points of $v^{-1}(A)$.
Let $Cl(v^{-1}(A)):=v^{-1}(A)\cup B(v^{-1}(A))$ denote the \textbf{closure} of $v^{-1}(A)$.
\end{defn}
Note that in the case $C$ has $k^{\prime}>k$ centroids, $C\in v^{-1}(A)$ implies
all centroids in $C\setminus \{c_{\pi(1)},\dots, c_{\pi(k)}\}$ are degenerate.
\paragraph{Representing $k$-means updates}
For now, let $C^*$ denote a ``local minimum'' of batch $k$-means and
suppose $C^*$ is not a boundary point. Let $A^*:=v^{-1}(C^*)$.
One run of batch $k$-means can be represented as
$$
C^0\mathrel{\mathop{\rightarrow}^{\mathrm{v(C^0)}}} A^1
\mathrel{\mathop{\rightarrow}^{\mathrm{m(A^1)}}} C^1
\mathrel{\mathop{\rightarrow}^{\mathrm{}}} \dots 
\mathrel{\mathop{\rightarrow}^{\mathrm{m(A^{t-1})}}} C^{t-1}
\mathrel{\mathop{\rightarrow}^{\mathrm{v(C^{t-1})}}} A^t \dots
\mathrel{\mathop{\rightarrow}^{\mathrm{}}} A^{*}
\mathrel{\mathop{\rightarrow}^{\mathrm{}}} C^*
$$
Figure \ref{fig:solution_spaces} illustrates how the algorithm alternates between
two solution spaces.
When batch $k$-means visits a clustering $A^t$, if $m(A^t)\notin Cl(v^{-1}(A^t))$, the algorithms jumps to another clustering $A^{t+1}$.
Otherwise, if $m(A^t)\in v^{-1}(A^t)$, the algorithm stops because $A^{t+1}=A^t$ and $m(A^{t+1})=m(A^t)$.
In the special case where $m(A^t)$ is a boundary stationary point,
since the algorithm arbitrarily breaks the tie, then it will continue to operate if the new clustering is chosen such that $A^{t+1}\neq A^t$, or stops if $A^{t+1}=A^t$.
In practice, it is unlikely to encounter a boundary point due to perturbations in the computing system, and regardless,
a sufficient condition for batch $k$-means to converge at the $t$-th iteration is $m(A^t)\notin Cl(v^{-1}(A^t))$. 
This motivates us to formalize ``local optima'' of batch $k$-means as below.
 \begin{defn}[Stationary clusterings]
We call $A^*$ a stationary clustering of $X$, if $m(A^*)\in Cl(v^{-1}(A^*))$.
We let $\{A^*\}_{[k]}\subset \{A\}_{[k]}$ denote the set of all stationary clusterings of $X$
with number of clusters $k^{\prime}\in [k]$.
\end{defn}
\begin{defn}[Stationary points]
For a stationary clustering $A^*$ with $k^{\prime}$ clusters, we define $C^{*}=\{c_r^*, r\in [k^{\prime}]\}$ to be a stationary point corresponding to $A^*$, so that $\forall A_{r}^{*}\in A^{*}$, $c_r^*:=m(A_r^{*})$.  
We let $\{C^*\}_{[k]}$ denote the corresponding set of all stationary points of $X$ with $k^{\prime}\in [k]$.
\end{defn}
\noindent
Note that the correspondence between $A^*$ and $C^*$ is one to one. And
by our definition, stationary points cannot be degenerate.
\paragraph{Distance function from $C^{\prime}$ to $C$}
Fix a clustering $A$ with its induced $k$-centroids $C:=m(A)$, and another set of $k^{\prime}$-centroids $C^{\prime}$ ($k^{\prime}\ge k$) with its induced clustering $A^{\prime}$, if
$|A^{\prime}|=|A|=k$ (this means if $k^{\prime}>k$, then $C^{\prime}$ has at least one degenerate centroid), then
we can pair the subset of non-degenerate $k$ centroids in $C^{\prime}$ with those in $C$, and ignore the degenerate centroids.
Under this condition, we define their centroidal distance as
$\Delta(C^{\prime},C):=\min_{\pi: [k]\rightarrow [k]}\sum_r n_r\|c^{\prime}_{\pi(r)}-c_{r}\|^2$. 
Sometimes we use $\Delta^t:=\Delta(C^t,C^{*})$ as a shorthand, when we want to measure the distance between $C^t$ from Algorithm \ref{alg:MBKM} and a fixed stationary point $C^*$.
This distance is \textit{asymmetric} and non-negative, and evaluates to zero if and only if two sets of centroids coincide. Using this distance function, we have a sufficient test on whether a clustering $A$ is stationary.

Note that the permutation $\pi^*:=\arg\min_{\pi: [k]\rightarrow [k]}\sum_r n_r\|c^{\prime}_{\pi(r)}-c_{r}\|^2$ may not be unique, unless $\Delta(C^{\prime},C)$ is small. In the rest of our paper when we refer to such a permutation that defines $\Delta(C^{\prime},C)$ in our proofs, it can be any permutation that is a minimizer.
The following lemma shows $\Delta$ distance can be used as a test to determine whether two sets of centroids induces the same clustering (belong to the same partition).
\begin{lm}
\label{lm:closure}
Fix a clustering $A=\{A_1,\dots, A_k\}$, let $C:=m(A)$ and $C^{\prime}\in v^{-1}(A)$, then 
$\exists \delta >0$ such that the following statement holds: 
\begin{eqnarray}
\label{statement}
\Delta(C^{\prime},C)<\delta \implies C\in v^{-1}(A)
\end{eqnarray}
\end{lm}
\begin{proof}
Since $C=m(A)$, $|C|=|A|=k$. 
Since $C^{\prime}\in v^{-1}(A)$, by Definition \ref{defn:members}, 
$|C^{\prime}|=k^{\prime}\ge k$, and
there is a permutation $\pi_{o}$ of $[k]$ and some subset
$\{c^{\prime}_{\pi_o(1)},\dots, c^{\prime}_{\pi_o(k)}\}\subset C^{\prime}$ such that
$\forall x\in A_r, \forall r\in [k]$,
$$
\|x - c^{\prime}_{\pi_{o}(r)}\|<\|x-c^{\prime}_{\pi_{o}(s)}\|, \forall s\ne r
$$
Note also the prerequisite for defining a distance $\Delta$ from $C^{\prime}$ to $C$ is satisfied, and we can choose $\delta > 0$ s.t. $\forall x\in A_r$, $\forall r\in [k], s\ne r$,
$$
\|x - c^{\prime}_{\pi_{o}(r)}\|\le\|x-c^{\prime}_{\pi_{o}(s)}\|-2\sqrt{\delta}
$$
and we require the equality is attained by at least one triple $(x,r,s)$.
Let $\pi^*$ be a permutation satisfying
$$
\pi^* = \arg\min_{\pi}\sum_{r\in [k]}n_r\|c^{\prime}_{\pi(r)}-c_r\|^2
$$
Let $\pi^{\prime}:=\pi_{*}^{-1}\circ\pi_{o}$. We have $\forall r, s\ne r$,
\begin{eqnarray*}
\|x-c_{\pi^{\prime}(s)}\| - \|x-c_{\pi^{\prime}(r)}\|
\ge
\|x-c^{\prime}_{\pi_{o}(s)}\|-\|c^{\prime}_{\pi_{o}(s)}-c_{\pi^{\prime}(s)}\|\\
-(\|x-c^{\prime}_{\pi_{o}(r)}\|+\|c^{\prime}_{\pi_{o}(r)}-c_{\pi^{\prime}(r)}\|)
>
\|x-c^{\prime}_{\pi_{o}(s)}\|-\|x-c^{\prime}_{\pi_{o}(r)}\|-2\sqrt{\delta} 
\ge 0
\end{eqnarray*}
where the second inequality is by the fact that
$$
\max_{r}\|c^{\prime}_{\pi_{o}(r)}-c_{\pi^{\prime}(r)}\|^2\le \Delta(C^{\prime},C)<\delta
$$
Since $\pi^{\prime}$ is the composition of two permutations of $[k]$, it is also a permutation of $[k]$, and 
$\forall r, s\ne r$, 
$\|x-c_{\pi^{\prime}(r)}\|<\|x-c_{\pi^{\prime}(s)}\|$, so $C\in v^{-1}(A)$ by Definition \ref{defn:members}.
\end{proof}
\paragraph{Remark:}
For any two consecutive iterations of batch $k$-means, $C^{t}, C^{t+1}$, if we let $A^{t+1}$ be $A$, then $C^{t}$ satisfies the condition for $C^{\prime}$ and $C^{t+1}$ for $C$. Applying Lemma \ref{lm:closure}, we can conclude that when $\Delta(C^t,C^{t+1})$ is sufficiently small, batch $k$-means converges.
Lemma \ref{lm:closure} also implies $\exists \delta>0$, such that the contrapositive of statement \eqref{statement} holds. This will be used in the proof of Lemma \ref{lm:limit_cluster}.
\paragraph{Equivalence of Algorithm \ref{alg:MBKM} to stochastic $k$-means}
Here, we formally show that Algorithm~\ref{alg:MBKM} with specific instantiation of sample size $m$ and learning rates $\eta^t_r$ is equivalent to online k-means \cite{BottouBengio} and mini-batch k-means \cite{Sculley}. 
\begin{claim}\label{claim:equivAlg}
In Algorithm~\ref{alg:MBKM}, if we set a counter for $\hat{N}^t_r:=\sum_{i=1}^t\hat{n}_r^i$ and
if we set the learning rate 
$\eta_r^t:=\frac{\hat{n}_r^t}{\hat{N}^t_r}$,
then provided the same random sampling scheme is used,
\begin{enumerate}
\item
When mini-batch size $m=1$, the update of Algorithm~\ref{alg:MBKM} is equivalent to that described in [Section 3.3, \cite{BottouBengio}].
\item
When $m>1$, the update of Algorithm~\ref{alg:MBKM} is equivalent to that described from line 3 to line 14 in [Algorithm 1, \cite{Sculley}] with mini-batch size $m$.
\end{enumerate}
\end{claim}
\begin{proof}
For the first claim, we first re-define the variables used in [Section 3.3, \cite{BottouBengio}]. We substitute index $k$ in \cite{BottouBengio} with $r$ used in Algorithm~\ref{alg:MBKM}. For any iteration $t$, we define the equivalence of definitions: $s \leftarrow x_i$, 
$c^t_r \leftarrow w_k$, $\hat{n}^t_r \leftarrow \Delta n_k$, $\hat{N}^t_r \leftarrow n_k$. 
According to the update rule in \cite{BottouBengio}, $\Delta n_k = 1$ if the sampled point $x_i$ is assigned to cluster with center $w_k$. Therefore, the update of the k-th centroid according to online $k$-means in \cite{BottouBengio} is:
$$
w_k\leftarrow w_k + \frac{1}{n_k}(x_i - w_k)1_{\{\Delta n_k =1\}}
$$
Using the re-defined variables, at iteration $t$, this is equivalent to 
$$
c^{t}_r = c^{t-1}_r + \frac{1}{\hat{N}^{t}_r}(s - c^{t-1}_r)1_{\{\hat{n}^t_r=1\}}
$$
Now the update defined by Algorithm~\ref{alg:MBKM} with $m=1$ and $\eta_r^t=\frac{\hat{n}_r^t}{\hat{N}_r^t}$ is:
\begin{eqnarray*}
c^{t}_r = c^{t-1}_r + \eta_r^t(\hat{c}_r^{t} - c^{t-1}_r)1_{\{\hat{n}^t_r\ne 0\}}\\
=c^{t-1}_r + \frac{\hat{n}_r^t}{\hat{N}_r^t}(s - c^{t-1}_r)1_{\{\hat{n}^t_r=1\}}
=c^{t-1}_r + \frac{1}{\hat{N}_r^t}(s - c^{t-1}_r)1_{\{\hat{n}^t_r=1\}}
\end{eqnarray*}
since $\hat{n}_r^t$ can only take value from $\{0,1\}$. This completes the first claim.

For the second claim, consider line 4 to line 14 in [Algorithm 1, \cite{Sculley}]. We substitute their index of time $i$ with $t$ in Algorithm~\ref{alg:MBKM}. We define the equivalence of definitions: $m\leftarrow b$, $S^t\leftarrow M$, $s \leftarrow x$, $c^{t-1}_{I(s)}\leftarrow d[x]$, $c^{t-1}_r\leftarrow c$. 

At iteration $t$, we let $v[c_r^{t-1}]_t$ denote the value of counter $v[c]$ upon completion of the loop from line 9 to line 14 for each center $c$, then $\hat{N}^t_{r}\leftarrow v[c_r^{t-1}]_t$. Since according to Lemma~\ref{techlm:avgAlg}, from line 9 to line 14, the updated centroid $c_r^{t}$ after iteration $t$ is
\begin{eqnarray*}
c_r^t = \frac{1}{v[c_r^{t-1}]_t}\sum_{s\in \cup_{i=1}^t S_r^{i}}s
= \frac{1}{\hat{N}^t_{r}}\sum_{s\in \cup_{i=1}^t S_r^{i}}s
\end{eqnarray*} 
This implies
\begin{eqnarray*}
c^t_r-c^{t-1}_r = \frac{1}{\hat{N}^t_r}\sum_{s\in\cup_{i=1}^t S_r^{i}}s -c^{t-1}_r\\
=\frac{1}{\hat{N}^t_r}[\sum_{s\in S^t_r}s+\sum_{s^{\prime}\in \cup_{i=1}^{t-1}S_r^{i}}s^{\prime} ]-c^{t-1}_r\\
=\frac{1}{\hat{N}^t_r}[\sum_{s\in S^t_r}s+\hat{N}^{t-1}_{r}c_r^{t-1}]-c^{t-1}_r\\
=-\frac{\hat{n}^t_r}{\hat{N}^{t}_r}c^{t-1}_r 
+ \frac{\hat{n}^t_r}{\hat{N}^{t}_r}\frac{\sum_{s\in S^t_r}s}{\hat{n}^t_r}
=-\eta_r^t c_r^{t-1}+\eta_r^t\hat{c}_r^t
\end{eqnarray*}
Hence, the updates in Algorithm~\ref{alg:MBKM} and line 4 to line 14 in [Algorithm 1, \cite{Sculley}] are equivalent.
\end{proof}
\section{Appendix B: proofs of main theorems}
One subtlety we need to point out before the proofs is that, in Algorithm \ref{alg:MBKM}, the learning rate $\eta_r^t$ as well as the update rule:
$$
c_r^t \leftarrow (1-\eta^t_r)c_r^{t-1} + \eta^t_r\hat{c}_r^t
$$
 is only defined for a cluster $r$ that is ``sampled'' at the $t$-th iteration. However, even if the cluster is not ``sampled'', i.e., $c_r^t = c_r^{t-1}$, the same update rule
with $\hat{c}_r^t=c_r^{t-1}$ and and the same learning rate still holds for this case.
So in our analysis, we equivalently treat each cluster $r$ as updated with the same learning rate $\eta^t=\frac{c^{\prime}}{t_o+t}$, and differentiates between a sampled  and not-sampled cluster only through the definition of $\hat{c}_r^t$. 
\subsection{Proofs leading to Theorem \ref{thm:km}}
\begin{proof}[Proof of Lemma \ref{lm:kmeans}]
For simplicity, we denote $E[\cdot|F_t]$ by $E_t[\cdot]$ (the same notation is also used as a shorthand to $E[\cdot|\Omega_t]$ in the proof of Theorem \ref{thm:mbkm_local}; we abuse the notation here).
\begin{eqnarray*}
E_t[\phi^{t+1}]
= E_t[\sum_{r=1}^k\sum_{x\in A_r^{t+2}}\|x-c_r^{t+1}\|^2]\\
\le E_t[\sum_r\sum_{x\in A_r^{t+1}}\|x-c_r^{t+1}\|^2] 
=E_t[\sum_r\sum_{x\in A_r^{t+1}}\|x-(1-\eta_r^{t+1})c_r^t-\eta_r^{t+1}\hat{c}_r^{t+1}\|^2]\\
=E_t[\sum_r\sum_{x\in A_r^{t+1}}(1-\eta_r^{t+1})^2\|x-c_r^t\|^2
+(\eta_r^{t+1})^2\|x-\hat{c}_r^{t+1}\|^2\\
+2\eta_r^{t+1}(1-\eta_r^{t+1})\langle x-c_r^t,x-\hat{c}_r^{t+1}\rangle]\\
\end{eqnarray*} 
where the inequality is due to the optimality of clustering $A^{t+2}$ for centroids $C^{t+1}$.
Since
$E_t[\hat{c}_r^{t+1}]=(1-p_r^{t+1})c_r^{t}+p_r^{t+1}m(A_r^{t+1})$, we have
$$
\langle x-c_r^t,x-\hat{c}_r^{t+1}\rangle
=
(1-p_r^{t+1})\|x-c_r^t\|^2
+p_r^{t+1} \langle x-c_r^t, x-m(A_r^{t+1}) \rangle
$$
Plug this into the previous inequality, we get
\begin{eqnarray*}
E_t[\phi^{t+1}]
\le
\sum_r(1-2\eta_r^{t+1})\phi^t_r
+(\eta_r^{t+1})^2\phi^t_r
+(\eta_r^{t+1})^2\sum_{x\in A_r^{t+1}}\|x-\hat{c}_r^{t+1}\|^2\\
+2\eta_r^{t+1}\{(1-p_r^{t+1})\sum_{x\in A_r^{t+1}}\|x-c_r^t\|^2
+p_r^{t+1} \sum_{x\in A_r^{t+1}}\langle x-c_r^t, x-m(A_r^{t+1}) \rangle\}\\
=\phi^t-2\sum_r\eta_r^{t+1}p_r^{t+1}\phi^t_r
+2\sum_r\eta_r^{t+1}p_r^{t+1}\sum_{x\in A_r^{t+1}}\langle x-c_r^t, x-m(A_r^{t+1}) \rangle\}\\
+(\eta_r^{t+1})^2\phi^t_r
+(\eta_r^{t+1})^2\sum_{x\in A_r^{t+1}}\|x-\hat{c}_r^{t+1}\|^2
\end{eqnarray*}
Note by Cauchy-Schwarz,
\begin{eqnarray*}
\sum_{x\in A_r^{t+1}}\langle x-c_r^t, x-m(A_r^{t+1}) \rangle
\le
\sqrt{\sum_{x\in A_r^{t+1}} \|x-c_r^t\|^2 \sum_{x\in A_r^{t+1}} \|x-m(A_r^{t+1})\|^2}\\
\le
\sum_{x\in A_r^{t+1}} \|x-c_r^t\|^2=\phi_r^t
\end{eqnarray*}
Now consider the two sums,
\begin{eqnarray*}
-2\sum_r\eta_r^{t+1}p_r^{t+1}\phi^t_r
+2\sum_r\eta_r^{t+1}p_r^{t+1}\sum_{x\in A_r^{t+1}}\langle x-c_r^t, x-m(A_r^{t+1}) \rangle\\
\end{eqnarray*}
For each $r$,
if $p_r^{t+1}=0$, both the left and right term are zero;
if $p_r^{t+1}>0$, we know the term is negative. Hence
\begin{eqnarray*}
-2\sum_r\eta_r^{t+1}p_r^{t+1}\phi^t_r
+2\sum_r\eta_r^{t+1}p_r^{t+1}\sum_{x\in A_r^{t+1}}\langle x-c_r^t, x-m(A_r^{t+1}) \rangle\\
\le
-2\min_{r,t; p_r^{t+1}>0}\eta_r^{t+1}p_r^{t+1}\sum_{r; p_r^{t+1}>0}(\phi^t_r - \sum_{x\in A_r^{t+1}}\langle x-c_r^t, x-m(A_r^{t+1}) \rangle)
\end{eqnarray*}
Our key observation is that $p_r^{t+1}=0$ if and only if cluster $A_r^{t+1}$ is empty, i.e.,
degenerate.
Since the degenerate clusters do not contribute to the $k$-means cost, 
we have $\sum_{r; p_r^{t+1}>0}\phi^t_r=\phi^t$. Moreover,
applying Cauchy-Schwarz again,
$$
\sum_{r; p_r^{t+1}>0}
\sqrt{\sum_{x\in A_r^{t+1}}\|x-c_r^t\|^2}
\sqrt{\sum_{x\in A_r^{t+1}}\|x-m(A_r^{t+1})\|^2}
\le
\sqrt{\phi^t\tilde\phi^{t}}
$$
Therefore,
\begin{eqnarray*}
E_t[\phi^{t+1}]
\le
\phi^t
-2\min_{r,t; p_r^{t+1}>0}\eta_r^{t+1}p_r^{t+1} (\phi^t - \sqrt{\phi^t\tilde\phi^t}) 
+(\eta_{\max}^{t+1})^2(E_t\sum_r\sum_{x\in A_r^{t+1}}\|x-\hat{c}_r^{t+1}\|^2+\phi^t)\\
=\phi^t
-2\min_{r,t; p_r^{t+1}>0}\eta_r^{t+1}p_r^{t+1} \phi^t(1 - \sqrt{\frac{\tilde\phi^t}{\phi^t}}) 
+(\eta_{\max}^{t+1})^2(E_t\sum_r\sum_{x\in A_r^{t+1}}\|x-\hat{c}_r^{t+1}\|^2+\phi^t)
\end{eqnarray*}
\end{proof}
\begin{proof}[Proof of Lemma \ref{lm:limit_cluster}]
First note that since $\{A^*\}_{[k]}$ includes all stationary clusterings of $1\le k^{\prime}\le k$
number of clusters. At any $t$, $C^t$ must have $k^{\prime}\in [k]$ non-degenerate centroids, so there exists $C^*=m(A^*)$ with $A^*\in\{A^*\}_{k^{\prime}}\in\{A^*\}_{[k]}$ such that
$
\Delta(C^t, C^*)
$ is well defined.
For a contradiction, suppose $\exists \delta>0$ such that
$\forall t$, $\Delta(C^t,C^*)>\delta$, for all $A^*\in\{A^*\}_{[k]}$.
At any $t\ge 1$, let $k^{\prime}$ denote the number of non-degenerate clusters in $C^t$, then
\paragraph{Case 1:}
If $A^{t+1}\in\{A^*\}_{k^{\prime}}$, then $\exists C^{*}=m(A^{t+1})$. Therefore, 
$$
\Delta(C^t, m(A^{t+1})) =\Delta(C^t,C^*)>\delta
$$ by our assumption.
\paragraph{Case 2:}
If $A^{t+1}\notin \{A^*\}_{k^{\prime}}$, then $C^{t+1}=m(A^{t+1})\notin Cl(v^{-1}(A^{t+1}))$. Then by the contrapositive of statement \eqref{statement} in Lemma \ref{lm:closure}, 
$$
\exists \delta^{\prime}>0 
\mbox{~s.t.~} \forall t, \Delta(C^t,C^{t+1})\ge \delta^{\prime}
$$
Combining the two cases, $\Delta(C^t, m(A^{t+1}))> \min\{\delta, \delta^{\prime}\}>0$.
By Lemma \ref{lm:kmeans} and assumption (A1), for any $t$,
$$
E[\phi^{t+1}-\phi^t|F_t]\le 
-\frac{2a\min_{\substack{r\in [k],t;p_r^t(m)>0}}p_r^t(m)}{t+t_o}\phi^t(1-\sqrt{\frac{\tilde\phi^t}{\phi^t}})
+(\frac{a}{t+t_o})^2B
$$
Since
$
\phi^t- \tilde\phi^t 
=\sum_{r\in [k^{\prime}]} \sum_{x\in A^{t+1}_r}\|x-C^t\|^2-\|x-m(A_r^{t+1})\|^2
= \sum_r \|c_r^t-m(A_r^{t+1})\|^2n_r^{t+1}
= \Delta(C^t, m(A^{t+1}))\ge \min\{\delta, \delta^{\prime}\}
 $, 
 we get
 $
1-\sqrt{\frac{\tilde\phi^t}{\phi^t}} \ge 1-\sqrt{1-\frac{\min\{\delta,\delta^{\prime}\}}{\phi^t}}
>0
 $.
Since $\min\{\delta, \delta^{\prime}\}$ is a constant and $\phi^t$ is upper bounded, the previous term is bounded away from zero for all $t$.
For convenience, let's denote by
$$
\min_{\substack{r\in [k],t;p_r^t(m)>0}}p_r^t(m)(1-\sqrt{\frac{\tilde\phi^t}{\phi^t}})
:= \frac{h}{\phi^{t}}>0
$$
Then $\forall t\ge 1$,
$$
E[\phi^{t+1}]-E[\phi^t] \le -\frac{2ah}{t+t_o}+\frac{a^2B}{(t+t_o)^2}
$$ for some positive constants $B, a, h$.
Summing up all inequalities,
$
E[\phi^{t+1}]-E[\phi^0]
\le -2ah\ln \frac{t+t_o+1}{t_o}+\frac{a^2B}{t_o-1}
$.
 Since $t$ is unbounded and $\ln \frac{t+t_o+1}{t_o}$ increases with $t$ while
 $\frac{a^2B}{t_o-1}$ is a constant, 
 $\exists T$ such that for all $t\ge T$,
 $E\phi^t - \phi^0 \le -\phi^0$, which means
 $E[\phi^{t}]\le 0$, for all $t$ large enough. 
 This implies the $k$-means cost of some clusterings is negative, which is impossible.
 So we have a contradiction. 
 \end{proof}
\begin{lm}
\label{lm:stat_stab}
If $\forall C^*\in \{C^*\}_{[k]}$, $C^*$ is a non-boundary stationary point, that is, $C^*:=m(A^*)\in v^{-1}(A^*)$. 
Then $\exists r_{\min}>0$ such that $\forall C^*\in \{C^*\}_{[k]}$, $C^*$ is a $(r_{\min},0)$-stable stationary point.
\end{lm}
\begin{proof}
Fix any $k$ in the range of $[k]$ (we abuse the notation with the same $k$ here).
For any $C$ such that $\Delta(C, C^*)$ exists (i.e., $|C|=k^{\prime}\ge k=|C^*|$), we first show $\exists r^*>0$, such that the following statement holds:
$$
\Delta(C, C^*) < r^*\phi^* \implies C \in v^{-1}(A^*)
$$
Thus, we proceed by argument analogous to that of Lemma \ref{lm:closure}.
By Definition \ref{defn:members}, 
there is a permutation $\pi_{o}$ of $[k]$ such that
$\forall x\in A_r, \forall r\in [k]$ and $\forall s\ne r$,
$$
\|x - c^*_{\pi_{o}(r)}\|<\|x-c^*_{\pi_{o}(s)}\|
$$
We choose $r^* > 0$ so that $\forall x\in A_r, \forall r\in [k], \forall s\ne r$,
$$
\|x - c^{*}_{\pi_{o}(r)}\|\le\|x-c^{*}_{\pi_{o}(s)}\|-2\sqrt{r^*\phi^*}, ~\forall r\in [k], s\ne r
$$
with equality holds for at least one triple of $(x, r, s)$.
Let $\pi^*$ be a permutation satisfying
$$
\pi^* = \arg\min_{\pi}\sum_{r\in [k]}n_r^*\|c_{\pi(r)}-c_r^*\|^2
$$
Let $\pi^{\prime}:=\pi_{*}\circ\pi_{o}$. We have $\forall (x,r,s)$ triples,
\begin{eqnarray*}
\|x-c_{\pi^{\prime}(s)}\| - \|x-c_{\pi^{\prime}(r)}\|
\ge
\|x-c^{*}_{\pi_{o}(s)}\|-\|c^{*}_{\pi_{o}(s)}-c_{\pi^{\prime}(s)}\|\\
-(\|x-c^{*}_{\pi_{o}(r)}\|+\|c^{*}_{\pi_{o}(r)}-c_{\pi^{\prime}(r)}\|)
>
\|x-c^*_{\pi_{o}(s)}\|-\|x-c^*_{\pi_{o}(r)}\|-2\sqrt{r^*\phi^*} \ge 0
\end{eqnarray*}
where the second inequality is by the fact that
$$
\max_{r}\|c_{\pi^*(r)}-c^*_{r}\|^2\le \Delta(C,C^*)<r^*\phi^*
\implies
\max_{r}\|c_{\pi^*(r)}-c^*_{r}\| <\sqrt{r^*\phi^*}
$$
Since $\pi^{\prime}$ is the composition of two permutations of $[k]$, it is also a permutation of $[k]$, and 
$\forall r, s\ne r$, 
$\|x-c_{\pi^{\prime}(r)}\|<\|x-c_{\pi^{\prime}(s)}\|$, so $C\in v^{-1}(A^*)$ by Definition \ref{defn:members}.
Since by our definition, $r^*$ is unique for each $C^*$. 
Since $\{C^*\}_{[k]}$ is finite, taking the minimum over all such $r^*$, i.e., $r_{\min}:=\min_{C^*\in \{C^*\}_{[k]}}r^*$ completes the proof.
\end{proof}
\paragraph{Remark:}
$r^*\phi^*$ is intuitively the ``radius'' or basin of attraction of a stationary point $C^*$. $r_{\min}$ is smallest such radius and is only determined by the geometry of the dataset.
\paragraph{Proof setup of Theorem \ref{thm:km}}
The goal of the proof is to show that first, (with high probability) the algorithm converges to some stationary clustering, $A^{**}\in \{A^*\}_{[k]}$. We call this event $G$; formally,
$$
G := \{\exists T\ge 1, \exists A^{**}\in \{A^*\}_{[k]}, \mbox{~s.t.~} A^t = A^{**}, \forall t\ge T\}
$$
Second, conditioning on $G$, we want to establish that the expected convergence rate of the algorithm to any stationary clustering $A^{**}$, which we formulate as
$$
E[\phi^t-\phi(A^{**})|G_o(A^{**})] 
$$ 
is of order $O(\frac{1}{t})$, where
$
G_o(A^{**})\subset G
$ is defined in the subsequent proof.

To prove the two claims, we consider an event $G_o\subset G\subset \Omega$, which we partition into disjoint sets $F(T^{-},A^{**})\cap H(T^{+},A^{**})$,
where by ``disjoint'' we mean disjoint w.r.t. different $(T,A^{**})$, for $T\ge 1$ and $A^{**}\in \{A^*\}_{[k]}$. 
Formally,
\begin{eqnarray*}
F(T^{-},A^{**}):=\{\forall t<T, \Delta(C^t, C^{**})>\frac{1}{2}r_{\min}\phi^{**}, \forall A^{**}\in \{A^*\}^t_{[k]}\}\\
\cap \{\Delta(C^T, C^{**})\le \frac{1}{2}r_{\min}\phi^{**}\}
\end{eqnarray*}
where we denote by $C^{**}:=m(A^{**})$, and we use $\{A^*\}^t_{[k]}$ to denote the subset of $\{A^*\}_{[k]}$ such that
$\Delta(C^t, C^{**})$ exists, 
and
$$
H(T^{+}, A^{**}):=\{\Delta(C^t, C^{**})\le r_{\min}\phi^{**}, \forall t\ge T\}
$$
Under the event $F(T^{-},A^{**})$, time $T$ is the first time the algorithm ``hits'' (hitting time) a stationary clustering, and this clustering is $A^{**}$. Note this event only depends on information of the sample space prior to and including time $T$, and we use $T^{-}$ to emphasize this.

Similarly, the event $H(T^{+}, A^{**})$ only depends on information post and including $T$.
Thus, for a fixed pair $(T, A^{**})$, $F(T^{-},A^{**})\cap H(T^{+}, A^{**})\ne \emptyset$; for $t\ne T$, the two events are complementary to each other, and for $t=T$, they intersect since $\{\Delta^T\le \frac{1}{2}r_{\min}\phi^{**}\}$ implies $\{\Delta^T\le r_{\min}\phi^{**}\}$.
Moreover,
fixing $T$, $F(T^{-},A^{**})$ are disjoint for different $A^{**}$;
fixing $A^{**}$, $F(T^{-},A^{**})$ are disjoint for different $T$.
Thus, the events $F(T^{-},A^{**})\cap H(T^{+}, A^{**})$ are disjoint for different pairs of $(T, A^{**})$.
\begin{proof}[Proof of Theorem \ref{thm:km}]
Now we define $G_o$ to be
$$
G_o := \cup_{T\ge 1}\cup_{A^{**}\in \{A^*\}_{[k]}} F(T^{-},A^{**})\cap H(T^{+}, A^{**})
$$
By Lemma \ref{lm:limit_cluster}
\begin{eqnarray}\label{eqn:sum_one}
Pr\{\cup_{T\ge 1}\cup_{A^{**}\in \{A^*\}_{[k]}}F(T^{-},A^{**})\} = 1
\end{eqnarray}
Conditioning on any $F(T^{-},A^{**})$, provided 
$$
c^{\prime}> \frac{1}{2\min_{\substack{r\in [k],t;\\p_r^t(m)>0}}p_r^t(m)(1-\sqrt{1-\frac{r_{\min}\phi^{opt}}{\phi^t}})}
\ge \frac{1}{2\min_{\substack{r\in [k],t;\\p_r^t(m)>0}}p_r^t(m)(1-\sqrt{1-\frac{r_{\min}\phi^{**}}{\phi^t}})}
$$
We apply Lemma \ref{lm:thm_part1} to get $\forall t \le T$,
$$
E\{\phi^t-\phi(A^{**})|F(T^{-},A^{**})\} = O(\frac{1}{t})
$$
Now let's consider the case $t\ge T$.
Since $A^{**}$ is $(r_{\min},0)$-stable, we can apply Theorem \ref{thm:mbkm_local};
specifically, for each parameters in the statement of Theorem \ref{thm:mbkm_local},
$\alpha=0$, 
$p_r^t(m) = 1- (\frac{n-n_r^{**}}{n})^m$,
$
a^t= \frac{\max_r p_r^t(m)}{\min_r p_r^t(m)}
$,
and
$
a_{\max}=\frac{c^{\prime}}{\min_{r,t}p_r^t(m)}
$.
Thus,  
$$
\beta := 2c^{\prime}(1-\max_t a^t\sqrt{\alpha})
=2c^{\prime} > 1
$$ is satisfied by setting 
$
c^{\prime}\ge \frac{1}{2}
$. 
Fix any $0<\delta<\frac{1}{e}$, the condition on $t_o$ in Theorem \ref{thm:mbkm_local} is satisfied by setting
$$
t_o\ge \max\left\{
\substack{
(\frac{16(c^{\prime})^2B}{[1-(1-p^{*}_{\min})^m]^2\Delta(C^T,C^{**})})^{\frac{1}{\beta-1}},\\
[\frac{48(c^{\prime})^2B^2}{(\beta-1)[1-(1-p^{*}_{\min})^m]^2\Delta(C^T,C^{**})^2}\ln\frac{1}{\delta}]^{\frac{2}{\beta-1}},
(\ln\frac{1}{\delta})^{\frac{2}{\beta-1}}
}\right\}
$$ 
Since $T$ is the hitting time of the event $\{\Delta(C^T, C^{**})\le \frac{1}{2}r_{\min}\phi^{**}\}$,
$\Delta(C^T, C^{**})\approx \frac{1}{2}r_{\min}\phi^{**}$, and we treat it as a constant.
So the conditions by our requirements on $c^{\prime}$ and $t_o$.
Let $\beta^{\prime}=2c^{\prime}$, $a_{\max}:=\frac{c^{\prime}}{\rho(m)}$.
Thus we can apply Theorem \ref{thm:mbkm_local} to get 
\begin{eqnarray}\label{eqn:conditional}
Pr\{H(T^{+},A^{**})|F(T^{-},A^{**})\}\ge 1-\delta
\end{eqnarray}
and $\forall t\ge T$,
\begin{eqnarray}
E\{\phi^t - \phi(A^{**}) |H(T^{+},A^{**})\}
\le E\{\Delta(C^t,C^{**})|H(T^{+},A^{**})\} \nonumber\\
\le
(\frac{t_o+T+1}{t_o+t})^{\beta^{\prime}}\Delta(C^T,C^{**}) + \frac{a_{\max}^2B}{\beta^{\prime}-1}
(\frac{t_0+T+2}{t_0+T+1})^{\beta^{\prime}+1}\frac{1}{t_0+t+1} 
 \nonumber\\
\le
(\frac{t_o+T+1}{t_o+t})^{\beta^{\prime}}r_{\min}\phi(A^{**}) + \frac{a_{\max}^2B}{\beta^{\prime}-1}
(\frac{t_0+T+2}{t_0+T+1})^{\beta^{\prime}+1}\frac{1}{t_0+t+1}
= O(\frac{1}{t})\label{eqn:more}
\end{eqnarray} 
where
the first inequality is by Lemma \ref{lm:kmdist_cdist}. 
Now observe
\begin{eqnarray*}
Pr\{G\}\ge Pr\{G_o\}
= Pr\{\cup_{T\ge 1}\cup_{A^{**}\in \{A^*\}_{[k]}} F(T^{-},A^{**})\cap H(T^{+}, A^{**})\}\\
= \sum_{T\ge 1,A^{**}\in \{A^*\}_{[k]}} Pr\{F(T^{-},A^{**})\cap H(T^{+}, A^{**})\}
\end{eqnarray*}
where the second equality holds because the events $F(T^{-},A^{**})\cap H(T^{+}, A^{**})$ are disjoint with respect to
different $(T, A^{**})$.
Since
\begin{eqnarray*}
\sum_{T\ge 1,A^{**}\in \{A^*\}_{[k]}} Pr\{F(T^{-},A^{**})\cap H(T^{+},A^{**})\}\\
=\sum_{T\ge 1,A^{**}\in \{A^*\}_{[k]}} Pr\{H(T^{+},A^{**})|F(T^{-},A^{**})\}Pr\{F(T^{-},A^{**})\}\\
 \ge (1-\delta)\sum_{T,A^{**}}Pr\{F(T,A^{**})\}
 =(1-\delta)Pr\{\cup_{T,A^{**}}F(T,A^{**})\}
 =(1-\delta)
\end{eqnarray*}
where the inequality is by Inequality \eqref{eqn:conditional}, and the last equality is due to Equality \eqref{eqn:sum_one}.
Therefore,
$
Pr\{G\}\ge 1-\delta
$,
which completes the proof of the first statement.

Let $G_o(A^{**}):=\cup_{T\ge 1} F(T^{-},A^{**})\cap H(T^{+}, A^{**})$, i.e., $G_o(A^{**})$
denotes the event where the algorithm converges to stationary clustering $A^{**}$, 
and
\begin{eqnarray*}
Pr\{\cup_{A^{**}\in\{A^*\}_{[k]}}G_{o}(A^{**})\}\\
= Pr\{\cup_{T\ge 1,A^{**}\in \{A^*\}_{[k]}}F(T^{-},A^{**})\cap H(T^{+}, A^{**})\}
\ge 1-\delta
\end{eqnarray*} which proves the second statement.
Finally, combining inequalities \eqref{eqn:less} and \eqref{eqn:more}, we have $\forall T\ge 1$ and $\forall t\ge 1$,
$$
E\{\phi^t - \phi(A^{**}) |F(T^{-},A^{**})\cap H(T^{+}, A^{**})\} = O(\frac{1}{t})
$$
Since the quantity $\phi^t - \phi(A^{**})$ is independent of information about $T$, we reach the conclusion
$$
E\{\phi^t - \phi(A^{**}) |G_{o}(A^{**})\} = O(\frac{1}{t})
$$
\end{proof}
\begin{lm}
\label{lm:thm_part1}
Suppose the assumptions and settings in Theorem \ref{thm:km} hold, 
conditioning on any $F(T^{-},A^{**})$, we have $\forall 1\le t < T$,
$$
E\{\phi^t-\phi(A^{**})|F(T^{-},A^{**})\} = O(\frac{1}{t})
$$
\end{lm}
\begin{proof}
First observe that conditioning on the event $F(T^{-},A^{**})$,
$
\Delta(C^t, C^{**})>\frac{1}{2}r_{\min}\phi^{**}
$,
$\forall t<T$.
Now we are in a setup similar to that in the proof Lemma \ref{lm:limit_cluster}, and the argument therein will lead us to the conclusion that
$$
\phi^t-\tilde\phi^t> \min\{\frac{1}{2}r_{\min}, r_{\min}\}\phi^{**}=\frac{1}{2}r_{\min}\phi^{**}
$$
By Lemma \ref{lm:kmeans}, we have conditioning on $F(T^{-},A^{**})$,
\begin{eqnarray}
\label{eqn:kmeans}
E_t[\phi^t|F(T^{-},A^{**})]
\le \phi^{t-1}\{1-\frac{2c^{\prime}\min_{r\in [k],t; p_r^t(m)>0}p_r^t(m)}{t_o+t-1}(1-\sqrt{\frac{\tilde\phi^t}{\phi^t}})\}\nonumber \\
+ \frac{(c^{\prime})^2B}{(t_o+t-1)^2}
\end{eqnarray} 
where $B$ is the bound in (A1).
Since by choosing
$$
c^{\prime}> \frac{1}{2\min_{\substack{r\in [k],t;\\p_r^t(m)>0}}p_r^t(m)(1-\sqrt{1-\frac{r_{\min}\phi^{opt}}{\phi^t}})}
\ge \frac{1}{2\min_{\substack{r\in [k],t;\\p_r^t(m)>0}}p_r^t(m)(1-\sqrt{1-\frac{r_{\min}\phi^{**}}{\phi^t}})}
$$
we have 
$$
\beta = 2c^{\prime}\min_{r\in [k],t; p_r^t(m)>0}p_r^t(m)(1-\sqrt{\frac{\tilde\phi^t}{\phi^t}})>1
$$
Then subtracting $\phi(A^{**})$ from both sides of Inequality \eqref{eqn:kmeans}
and applying Lemma \ref{lm:tech_2} with $a:=\beta>1$, we conclude that $\forall 1\le t<T$,
\begin{eqnarray}
E[\phi^t - \phi(A^{**})|F(T^{-},A^{**})]
\le \frac{t_o+1}{t_o+t}[\phi^0-\phi(A^{**})]\nonumber\\
+\frac{(c^{\prime})^2B}{\beta-1}(\frac{t_o+2}{t_o+1})^{\beta+1}\frac{1}{t_o+t+1}
=O(\frac{1}{t})\label{eqn:less}
\end{eqnarray}
\end{proof}
\paragraph{The overall exact convergence bound}
There are two convergence bound corresponding to the two phases of analysis, global and local convergence, in Theorem \ref{thm:km}.
For the first global convergence phase from $t=1$ to $T$ for some random $T$, the bound is 
$$
(\frac{t_o+1}{t_o+t})^{\beta}[\phi^0-\phi(A^{**})]
+\frac{(c^{\prime})^2B}{\beta-1}(\frac{t_o+2}{t_o+1})^{\beta+1}\frac{1}{t_o+t+1}
\mbox{\quad (by Lemma \ref{lm:thm_part1})}
$$
where $\beta=2c^{\prime}p_{\min}^{+}(m)(1-\sqrt{\frac{\tilde\phi^t}{\phi^t}})$.
For the second, local convergence phase, $\forall t\ge T$,
\begin{eqnarray*}
(\frac{t_o+T+1}{t_o+t})^{\beta^{\prime}}r_{\min}\phi(A^{**}) + \frac{a_{\max}^2B}{\beta^{\prime}-1}
(\frac{t_0+T+2}{t_0+T+1})^{\beta^{\prime}+1}\frac{1}{t_0+t+1}
\mbox{~~(by Theorem \ref{thm:mbkm_local})}
\end{eqnarray*}
where $\beta^{\prime}=2c^{\prime}$, $a_{\max}:=\frac{c^{\prime}}{\rho(m)}$.
So the overall bound is the maximum of the two bounds.
When $c^{\prime}$ is sufficiently large, the first term of both bounds are likely dominated by the second term, so the overall bound is likely dominated by 
$
\frac{a_{\max}^2B}{\beta-1}
(\frac{t_0+2}{t_0+1})^{\beta+1}\frac{1}{t_0+t+1}
$. 
\subsection{Proofs leading to Theorem \ref{thm:mbkm_local}}
In the analysis of this section, we use $E_t[\cdot]$ as a shorthand notation for $E[\cdot|\Omega_t]$, where $\Omega_t$ is as defined in the main paper.
\begin{proof}[Proof of Theorem \ref{thm:mbkm_local}]
By Proposition \ref{prop:high_prob}, for any $t>1$,
$
Pr(\Omega_t)\ge 1-Pr(\cup_{t>i}\Omega_{t-1}\setminus\Omega_t)
\ge 1-\delta
$. This proves the first statement.
By Lemma \ref{lm:iter_wise},
\begin{eqnarray*}
E_t[\Delta^t|F_{t-1}]
\le 
\Delta^{t-1}(1-\frac{\beta}{t_o+t})
+[\frac{a_{\max}}{t_o+t}]^2B
+\frac{2a_{\max}}{t_o+t}E_t\{\sum_{r}n_r^*\langle c_r^{t-1}-c_r^*,\xi_r^t\rangle|F_{t-1}\}\\
=
\Delta^{t-1}(1-\frac{\beta}{t_o+t})
+[\frac{a_{\max}}{t_o+t}]^2B
\end{eqnarray*}
where the last term vanishes since $E_t\{\xi_r^t|F_{t-1}\}=0$, $\forall r\in [k]$.
Therefore,
$$
E_t[\Delta^t]\le E_{t-1}[\Delta^{t-1}](1-\frac{\beta}{t_o+t})+\frac{a_{\max}^2B}{(t+t_o)^2}
\le
\dots
\le
E_i[\Delta^i]\Pi_{\tau=i+1}^t(1-\frac{\beta}{t_o+\tau})+\sum_{\tau=i+1}^t\frac{a_{\max}^2B}{(\tau+t_o)^2}
$$
By Lemma \ref{lm:tech_2}, for $\beta>1$,
$$
E_t[\Delta^t]\le
(\frac{t_o+i+1}{t_o+t})^{\beta}\Delta^i + \frac{a_{\max}^2B}{\beta-1}
(\frac{t_0+i+2}{t_0+i+1})^{\beta+1}\frac{1}{t_0+t+1}
$$
\end{proof}
\begin{lm}
\label{lm:iter_wise}
Suppose the conditions in Theorem \ref{thm:mbkm_local} hold.
If $\frac{c}{t_o+i}\le\eta^i p_r^i(m) \le \frac{ca}{t_o+i}, \forall r\in [k]$, with
$a\le \frac{1}{\sqrt{\alpha}}$, then
\begin{eqnarray*}
\Delta^i\le \Delta^{i-1}(1-\frac{\beta}{t_o+i})
+[\frac{a_{\max}}{t_o+i}]^2\sum_r n_r^*\|\hat{c}_r^i-c_r^* \|^2
+\frac{2a_{\max}}{t_o+i}\sum_{r}n_r^*\langle c_r^{i-1}-c_r^*,\xi_r^i\rangle
\end{eqnarray*}
where $\xi_r^i := \hat{c}_r^i-E[\hat{c}_r^i|F_{i-1}]$.
\end{lm}
\begin{proof}
Let $\Delta^i_r:=n_r^*\|c_r^i-c_r^*\|^2$, so $\Delta^i = \sum_r\Delta^i_r$.
By the update rule of Algorithm \ref{alg:MBKM},
\begin{eqnarray*}
\Delta^i_r
=n_r^*\|(1-\eta^i)(c_r^{i-1}-c_r^*)+\eta^i(\hat{c}_r^i-c_r^*)\|^2\\
\le n_r^*\{(1-2\eta^i)\|c_r^{i-1}-c_r^*\|^2
+2\eta^i\langle c_r^{i-1}-c_r^*,\hat{c}_r^i-c_r^* \rangle\\
+(\eta^i)^2[\|c_r^{i-1}-c_r^*\|^2+\|\hat{c}_r^i-c_r^*\|^2]\}
\end{eqnarray*}
Let $\xi_r^i = \hat{c}_r^i-E[\hat{c}_r^i|F_{i-1}]$, where
$E[\hat{c}_r^i|F_{i-1}] = (1-p_r^i)c_r^{i-1}+p_r^im(A_r^{i})$,
and $p_r^i:=p_r^i(m)$.

Since
\begin{eqnarray*}
\langle c_r^{i-1}-c_r^*,\hat{c}_r^i-c_r^* \rangle
=\langle c_r^{i-1}-c_r^*,E[\hat{c}_r^i|F_{i-1}]+\xi_r^i-c_r^* \rangle\\
\le
(1-p_r^i)\|c_r^{i-1}-c_r^*\|^2 + p_r^i\|m(A_r^{i})-c_r^*\|\|c_r^{i-1}-c_r^*\| 
+ \langle c_r^{i-1}-c_r^*,\xi_r^i \rangle
\end{eqnarray*}
we have
\begin{eqnarray*}
\Delta_r^i 
\le
n_r^*\{
-2\eta^i[\|c_r^{i-1}-c_r^*\|^2-(1-p_r^i)\|c_r^{i-1}-c_r^*\|^2
- p_r^i\|c_r^{i-1}-c_r^*\|\|m(A_r^i)-c_r^*\|]\\
+\|c_r^{i-1}-c_r^*\|^2
+2\eta^i\langle \xi_r^i, c_r^{i-1}-c_r^*\rangle
+(\eta^i)^2[\|c_r^{i-1}-c_r^*\|^2+\|\hat{c}_r^i-c_r^*\|^2]
\}\\
\le
n_r^*\{
-\frac{2c}{t_o+i}\|c_r^{i-1}-c_r^*\|^2
+\frac{2ca}{t_o+i}\|c_r^{i-1}-c_r^*\|\|m(A_r^i)-c_r^*\|\\
+\|c_r^{i-1}-c_r^*\|^2+2\eta^i\langle \xi_r^i, c_r^{i-1}-c_r^*\rangle
+(\eta^i)^2[\|c_r^{i-1}-c_r^*\|^2+\|\hat{c}_r^i-c_r^*\|^2
\}
\end{eqnarray*}
Note
\begin{eqnarray*}
\sum_{r}n_r^*\|c_r^i-c_r^*\|\|m(A_r^i)-c_r^*\|
\le
\sqrt{
(\sum_r n_r^* \|c_r^{i-1}-c_r^*\|^2)
(\sum_r n_r^* \|m(A_r^i)-c_r^*\|^2)
}\\
=\sqrt{\Delta^{i-1}\Delta(m(A^i),C^*)}
\le
\sqrt{\alpha}\Delta^{i-1}
\end{eqnarray*}
where the first inequality is by Cauchy-Schwartz and 
the last inequality is by applying Lemma \ref{lm:stable_stat}.
Finally, summing over $\Delta_r^i$, we get
\begin{eqnarray*}
\Delta^i
=\sum_r\Delta^i_r
\le \Delta^{i-1}[1-\frac{2c}{t_o+i}(1-a\sqrt{\alpha})]\\
+ [\frac{ca}{(t_o+i)p_r^i}]^2\sum_r n_r^*\|\hat{c}_r^i-c_r^* \|^2
+ \frac{2ca}{(t_o+i)p_r^i}\sum_r n_r^*\langle c_r^{i-1}-c_r^*, \xi_r^i\rangle\\
\le
\Delta^{i-1}(1-\frac{\beta}{t_o+i})
+[\frac{a_{\max}}{t_o+i}]^2\sum_r n_r^*\|\hat{c}_r^i-c_r^* \|^2
+\frac{2a_{\max}}{t_o+i}\sum_r n_r^*\langle c_r^{i-1}-c_r^*, \xi_r^i\rangle\\
\end{eqnarray*}
The second inequality is by definition of $\beta:=2c^{\prime}\min_{r,t}p_r^t(m)(1-\max_{t}a^t\sqrt{\alpha}))$,
and the fact that $c\ge c^{\prime}\min_{r,t}p_r^t(m)$ and $a\le \max_{t}a^t$.
\end{proof}
\begin{lm}
\label{lm:conditional}
Suppose the conditions in Theorem \ref{thm:mbkm_local} hold.
For any $\lambda>0$, 
$$
E_i\{\exp\{\lambda \Delta^i\}|F_{i-1}\}
\le \exp\left\{\lambda\{(1-\frac{\beta}{t_0+i})\Delta^{i-1}
+\frac{a_{\max}^2 B}{(t_0+i)^2}
+\frac{\lambda a_{\max}^2 B^2}{2(t_0+i)^2}\}
\right\}
$$
\end{lm}
\begin{proof}
By Lemma \ref{lm:iter_wise}, we have
\begin{eqnarray*}
E_i\{\exp(\lambda \Delta^i)|F_{i-1}\}
\le \exp \lambda [\Delta^{i-1}(1-\frac{\beta}{t_o+i})+\frac{a_{\max}^2B}{(t_o+i)^2}]\\
E_i\{\exp\lambda\frac{2a_{\max}}{t_o+i}\sum_{r}n_r^*\langle c_r^{i-1}-c_r^*,\xi_r^i\rangle|F_{i-1}\}
\end{eqnarray*}
Since $\frac{2\lambda a_{\max}}{i+t_0} \sum_{r}n_r^*\langle \xi_r^i, c_r^{i-1}-c_r^{*}\rangle 
\le \frac{2\lambda a_{\max}}{i+t_0} B$ and
$E_i\{\frac{2\lambda a_{\max}}{i+t_0} \sum_{r}n_r^*\langle \xi_r^i, c_r^{i-1}-c_r^{*}\rangle|F_{i-1}\} = 0$,
by Hoeffding's lemma
$$
E_i\left\{\exp\{\frac{2\lambda a_{\max}}{i+t_0}\sum_{r}n_r^*\langle \xi_r^i, c_r^{i-1}-c_r^{*}\rangle|F_{i-1}\}
\right\} 
\le \exp\{\frac{\lambda_{\max}^2 a_{\max}^2B^2}{2(i+t_0)^2}\}
$$
Combining this with the previous bound completes the proof.
\end{proof}
\begin{lm}[adapted from \cite{balsubramani13}]
\label{lm:bal}
For any $\lambda >0$,
$E_i\{e^{\lambda \Delta^{i-1}}\}\le E_{i-1}\{e^{\lambda \Delta^{i-1}}\}$
\end{lm}
\begin{proof}
By our partitioning of the sample space, $\Omega_{i-1}=\Omega_i\cup (\Omega_{i-1}\setminus \Omega_i)$,
and for any $\omega \in \Omega_i$ and  $\omega^{\prime} \in \Omega_{i-1}\setminus \Omega_i$,
$\Delta^{i-1}(\omega) \le \Delta^{0}<\Delta^{i-1}(\omega^{\prime})$. 
Taking expectation
over $\Omega_i$ and $\Omega_{i-1}$, we get
$
E_i\{e^{\lambda \Delta^{i-1}}\}\le E_{i-1}\{e^{\lambda \Delta^{i-1}}\}
$.
\end{proof}
\begin{prop}
\label{prop:high_prob}
Suppose the conditions in Theorem \ref{thm:mbkm_local} hold.
Fix any $0<\delta\le \frac{1}{e}$.
If
$$
t_0\ge 
\max\{
(\frac{16a_{\max}^2B}{\Delta^0})^{\frac{1}{\beta-1}},
[\frac{48a_{\max}^2B^2}{(\beta-1)(\Delta^0)^2}\ln\frac{1}{\delta}]^{\frac{2}{\beta-1}},
(\ln\frac{1}{\delta})^{\frac{2}{\beta-1}}
\}
$$ 
then
$P(\cup_{i\ge 1}\Omega_{i-1}\setminus \Omega_{i})\le \delta$
(here we used $\Delta^0$ instead of $\Delta^i$ and treat the starting time, the $i$-th iteration in Theorem \ref{thm:mbkm_local} as the zeroth iteration for cleaner presentation).
\end{prop}
\begin{proof}
By Lemma \ref{lm:conditional},
\begin{eqnarray*}
E_i\{e^{\lambda \Delta^i}\}
\le 
E_i\{e^{\lambda\{(1-\frac{\beta}{t_{o}+i})\Delta^{i-1}}\}
\exp\{
\frac{\lambda a_{\max}^2 B}{(t_o+i)^2}
+\frac{\lambda^2 a_{\max}^2 B^2}{2(t_o+i)^2}\}\\
\le
E_{i-1}\{e^{\lambda^{(1)}\Delta_r^{i-1}}\}
\exp\{
\frac{\lambda a_{\max}^2 B}{(t_o+i)^2}
+\frac{\lambda^2 a_{\max}^2 B^2}{2(t_o+i)^2}\}
\end{eqnarray*} 
where $\lambda^{(1)} = \lambda(1-\frac{\beta}{t_{o}+i})$, and the second inequality
is by Lemma \ref{lm:bal}.
Similarly, the following recurrence relation holds for $k=0,\dots,i$:
\begin{eqnarray*}
E_{i-k}\{e^{\lambda^{(k)} \Delta^{i-k}}\}
\le
E_{i-(k+1)}\{e^{\lambda^{(k+1)}\Delta^{i-k-1}}\}
\exp\{
\frac{\lambda^{(k)} a_{\max}^2 B}{(t_o+i-k)^2}
+\frac{(\lambda^{(k)})^2 a_{\max}^2 B^2}{2(t_o+i-k)^2}
\}
\end{eqnarray*}
where $\lambda^{(0)}:=\lambda$, 
and for $k\ge 1$, $\lambda^{(k)}:=\Pi_{t=1}^{k}(1-\frac{2}{t_o+(i-t+1)})\lambda^{(0)}$.

Since (see, e.g., \cite{balsubramani13})
$\forall \beta>0, k\ge 1$, 
$$
\Pi_{t=1}^{k}(1-\frac{\beta}{t_o+(i-t+1)})\le (\frac{t_o+i-k+1}{t_o+i})^{\beta}
$$
we have, for $0\le k\le i$ (this inequality is also satisfied when $k=0$ by definition of $\lambda^{(0)}$),
$$
\frac{\lambda^{(k)}}{(t_0+i-k)^2}
\le (\frac{t_o+i-k+1}{(t_o+i)(t_o+i-k)})^2\lambda
\le 
(1+\frac{2}{t_o}+\frac{1}{t_o^2})\frac{\lambda}{(t_o+i)^{\beta}}
$$ 
Repeatedly applying the relation, we get
\begin{eqnarray*}
E_i\{e^{\lambda \Delta^i}\}
\le e^{\lambda^{(i)}\Delta^{0}}
\exp\{\sum_{k=0}^{i-1}(\frac{\lambda^{(k)} a_{\max}^2 B}{(t_o+i-k)^2}
+\frac{(\lambda^{(k)})^2 a_{\max}^2 B^2}{2(t_o+i-k)^2})\}\\
\le 
\exp\{
\lambda(\frac{t_o+1}{t_o+i})^{\beta}\Delta^{0}
+(\lambda B+\frac{\lambda^2B^2}{2})\frac{(1+\frac{2}{t_o}+\frac{1}{t_o^2})a_{\max}^2}{(t_o+i)^{\beta-1}}
\}
\end{eqnarray*}
Then by Markov's inequality, for any $\lambda_i>0$,
$$
Pr(\omega \in \Omega_{i-1}\setminus\Omega_i)
\le Pr_{i}(\Delta^i>2\Delta^0)
=Pr_{i}(e^{\lambda_i\Delta^i}>e^{\lambda_i2\Delta^0})
\le 
\frac{E_{i}e^{\lambda_i\Delta_r^i}}{e^{\lambda_i2\Delta_r^0}}
$$
Combining this with the upper bound on $E_{i}e^{\lambda_i\Delta^i}$, we get 
\begin{eqnarray*}
Pr(\omega \in \Omega_{i-1}\setminus\Omega_i)
\le 
\exp\left\{-\lambda_i
\{
\Delta^{0}[2-(\frac{t_o+1}{t_o+i})^{\beta}]
-(B+\frac{\lambda_i B^2}{2})\frac{(1+\frac{2}{t_o}+\frac{1}{t_o^2})a_{\max}^2}{(t_o+i)^{\beta-1}}
\}
\right\}\\
\le
\exp\left\{-\lambda_i
\{
\Delta^{0}
-(B+\frac{\lambda_i B^2}{2})\frac{4a_{\max}^2}{(t_o+i)^{\beta-1}}
\}
\right\}
\end{eqnarray*} 
since $i\ge 1$, $t_o\ge 1$. 
We choose $\lambda_i = \frac{1}{\Delta}\ln \frac{(i+1)^2}{\delta}$ with
$\Delta = \frac{\Delta^{0}}{2}$.
\paragraph{Case 1:}
$B>\frac{\lambda_i B^2}{2}$. We get
$
\Delta^{0}
-(B+\frac{\lambda_i B^2}{2})\frac{4a_{\max}^2}{(t_o+i)^{\beta-1}}
\ge \Delta
$, since $t_o \ge (\frac{16a_{\max}^2B}{\Delta^0})^{\frac{1}{\beta-1}}$.
\paragraph{Case 2:} $B\le \frac{\lambda_i B^2}{2}$. We get
\begin{eqnarray*}
\Delta^0
-(B+\frac{\lambda_i B^2}{2})\frac{4a_{\max}^2}{(t_o+i)^{\beta-1}}
\ge
\Delta^0
-\lambda_i B^2\frac{4a_{\max}^2}{(t_o+i)^{\beta-1}}\\
=
2\Delta-\frac{1}{\Delta}\ln \frac{(1+i)^2}{\delta}\frac{4a_{\max}^2B^2}{(t_o+i)^{\beta-1}}
\ge
2\Delta-\frac{1}{\Delta}\ln \frac{(t_o+i)^2}{\delta}\frac{4a_{\max}^2B^2}{(t_o+i)^{\beta-1}}
\end{eqnarray*}
Now we show 
$$
\frac{1}{\Delta}\ln \frac{(t_o+i)^2}{\delta}\frac{4a_{\max}^2B^2}{(t_o+i)^{\beta-1}}
\le \Delta
$$
Since 
$t_o+i\ge t_o
\ge
\max\{[\frac{48a_{\max}^2B^2}{(\beta-1)(\Delta^0)^2}\ln\frac{1}{\delta}]^{\frac{2}{\beta-1}},
(\ln\frac{1}{\delta})^{\frac{2}{\beta-1}}
\}
$,
we can apply Lemma \ref{lm:tech} with 
$C:=\max\{
\frac{16a_{\max}^2B^2}{(\Delta^0)^2},
\frac{\beta-1}{3}
\}$, 
$t := t_o+i \ge (\frac{3C}{\beta-1}\ln\frac{1}{\delta})^{\frac{2}{\beta-1}}$, and get 
$$
\frac{4a_{\max}^2B^2}{\Delta^2}\ln\frac{(t_o+i)^2}{\delta}
=(\ln \frac{(t_o+i)^2}{\delta})\frac{16a_{\max}^2B^2}{(\Delta_r^0)^2}
\le 2C\ln t +C\ln\frac{1}{\delta} <t^{\beta-1} = (t_o+i)^{\beta-1}
$$
That is,
$\frac{1}{\Delta}\ln \frac{(t_o+i)^2}{\delta}\frac{4a_{\max}^2B^2}{(t_o+i)^{\beta-1}}
\le \Delta
$.
Thus, for both cases, 
$
Pr(\omega \in \Omega_{i-1}\setminus\Omega_i)
\le e^{-\frac{1}{\Delta}(\ln\frac{(1+i)^2}{\delta})\Delta}
\le \frac{\delta}{(i+1)^2}
$.
Finally, we have
$
\sum_{i=1}^{\infty}Pr(\omega \in \Omega_{i-1}\setminus\Omega_i)
\le \delta.
$
\end{proof}
\section{Appendix C: Proofs leading to Theorem \ref{thm:solution}}
\subsection{Existence of stable stationary point under geometric assumptions on the dataset}
\begin{lm}[Theorem 5.4 of \cite{kumar}]
\label{lm:kumar}
Suppose the dataset $X$ admits a stationary clustering $A^*$ satisfying $(B1)$ and $(B2)$.
If $\forall r\in [k], s\ne r$, $\Delta_r^t+\Delta_s^t\le \frac{\Delta_{rs}}{16}$. Then for any $s\ne r$,
$
|A_r^*\cap A_s^t| \le \frac{b^2}{f}
$, where $b\ge \max_{r,s}\frac{\Delta_r^t+\Delta_s^t}{\Delta_{rs}}$.
\end{lm}
The proof is almost verbatim to Theorem 5.4 of \cite{kumar}; we include it here for completeness.
\begin{proof}
Let $x\in A^{*}_r\cap A^{t}_s$. Split $x$ into its projection on the line joining $c^{*}_r$ and $c^{*}_s$, and its orthogonal component.
This implies
\begin{eqnarray*}
x = \frac{1}{2}(c^{*}_r+c^{*}_s)+\lambda (c^{*}_r-c^{*}_s)+u
\end{eqnarray*}
with $u\perp c^{*}_r-c^{*}_s$. 
Note $\lambda$ measures the degree of departure of the projected point to the mid-point of $c^{*}_r$ and $c^{*}_s$. 
Thus, by our definition of the margin, we have
\begin{eqnarray}
\label{eqn:margin}
\|\bar{x}-\frac{1}{2}(c^{*}_r+c^{*}_s)\|=\|\lambda(c^{*}_r-c^{*}_s)\|
\ge \frac{1}{2}\Delta_{rs} 
\end{eqnarray}
Since the projection of $x$ on the line joining $c^t_r, c^t_s$ is closer to $s$, we have
\begin{eqnarray*}
x(c^t_s -c^t_r)\ge \frac{1}{2}(c^t_s-c^t_r)(c^t_s+c^t_r)
\end{eqnarray*}
So
\begin{eqnarray}
\frac{1}{2}(c^{*}_r+c^{*}_s)(c^t_s-c^t_r)
+\lambda(c^{*}_r-c^{*}_s)(c^t_s-c^t_r)\nonumber\\
+u(c^t_s-c^t_r)\ge \frac{1}{2}(c^t_s-c^t_r)(c^t_s+c^t_r)\label{eqn:goodpoints}
\end{eqnarray}
Since $u\perp c^{*}_r-c^{*}_s$, let $\Delta=\Delta^t_s+\Delta^t_r$. We have
$$
u(c^t_s-c^t_r) = u(c^t_s-c^{*}_s-(c^t_r-c^{*}_r))
\le \|u\|\Delta
$$
Rearranging Inequality \eqref{eqn:goodpoints}, we have
\begin{eqnarray*}
 \frac{1}{2}(c^{*}_r+c^{*}_s-c^t_s-c^t_r)(c^t_s-c^t_r)\nonumber\\
 +\lambda(c^{*}_r-c^{*}_s)(c^t_s-c^t_r)
+u(c^t_s-c^t_r)\ge 0\nonumber\\
\equiv
\frac{\Delta^2}{2}+\frac{\Delta}{2}\|c^{*}_r-c^{*}_s\|
-\lambda\|c^{*}_r-c^{*}_s\|^2\nonumber\\
+\lambda\Delta \|c^{*}_r-c^{*}_s\|+\|u\|\Delta\ge 0
\end{eqnarray*}
Therefore,
\begin{eqnarray*}
\|x - c^{*}_r\| = \|(\frac{1}{2}-\lambda)(c^{*}_s-c^{*}_r)+u\|\ge \|u\| \\
\ge \frac{\lambda}{\Delta}\|c^{*}_r-c^{*}_s\|^2-\frac{\Delta}{2}\\
-\frac{1}{2}\|c^{*}_r-c^{*}_s\|
-\lambda \|c^{*}_r-c^{*}_s\| \label{ineq:1}
\ge \frac{\Delta_{rs}\|c^*_r-c^*_s\|}{64\Delta}
\end{eqnarray*}
where the last inequality is by our assumption that $\Delta\le\frac{\Delta_{rs}}{16}$, 
and $\lambda \ge \frac{\Delta_{rs}}{2\|c^{*}_r-c^{*}_s\|}$ by \eqref{eqn:margin}.
Therefore, for all $s\ne r$
\begin{eqnarray*}
|A^{*}_r\cap A^{t}_s|\frac{\Delta^2_{rs}\|c_r^*-c_s^*\|^2}{f\Delta^2} 
\le \sum_{A^{*}_r\cap A^t_s}\|x-c^{*}_r\|^2
\end{eqnarray*}
So
$
|A^{*}_r\cap A^{t}_s|\le \sum_{A^{*}_r\cap A^t_s}\|x-c^{*}_r\|^2
\frac{f(\Delta^t_r+\Delta^t_s)^2}{\Delta_{rs}^2\|c_r^*-c_s^*\|^2}
\le\frac{f b^2}{f^2\phi^*(\frac{1}{n^*_r})}( \sum_{A^{*}_r\cap A^t_s}\|x-c^{*}_r\|^2)
$. That is,
$
\frac{|A^{*}_r\cap A^{t}_s|}{n^*_r}
\le \frac{b^2}{f\phi^*} \sum_{A^{*}_r\cap A^t_s}\|x-c^{*}_r\|^2
$.
Similarly, for all $s\ne r$,
$
\frac{|A^{*}_s\cap A^{t}_r|}{n^*_r}
\le \frac{b^2}{f\phi^*} \sum_{A^{*}_s\cap A^t_r}\|x-c^{*}_s\|^2
$
Summing over all $s\ne r$,
$\frac{|A_{r}\triangle A_r^*|}{n_r^*}
=\rho_{out}+\rho_{in}\le \frac{b^2}{f\phi^*}\phi^* = \frac{b^2}{f}$.
\end{proof}
\begin{lm}
\label{lm:kmdist_cdist}
Fix a stationary point $C^*$ with $k$ centroids, and any other set of $k^{\prime}$-centroids, $C$, with $k^{\prime}\ge k$ so
that $C$ has exactly $k$ non-degenerate centroids.
We have
$$
\phi(C)-\phi^*\le \min_{\pi}\sum_{r}n_r^*\|c_{\pi(r)}-c_r^*\|^2 = \Delta(C,C^*)
$$
\end{lm}
\begin{proof}
Since degenerate centroids do not contribute to $k$-means cost, in the following we only consider
the sets of non-degenerate centroids $\{c_s, s\in [k]\}\subset C$ and $\{c_r^*, r\in [k]\}\subset C^*$.
We have for any permutation $\pi$,
\begin{eqnarray*}
\phi(C)-\phi^*
=\sum_s\sum_{x\in A_s}\|x-c_s\|^2-\sum_r\sum_{x\in A^*_r}\|x-c^*_r\|^2\\
\le
\sum_r\sum_{x\in A^*_r}\|x-c_{\pi(r)}\|^2-\sum_r\sum_{x\in A^*_r}\|x-c^*_r\|^2
= \sum_r n_r^*\|c_{\pi(r)}-c_r^*\|^2
\end{eqnarray*}
where the last inequality is by optimality of clustering assignment based on Voronoi diagram,
and the second inequality is by applying the centroidal property in Lemma \ref{centroidal} to each centroid in $C^*$.
Since the inequality holds for any $\pi$, it must holds for 
$\min_{\pi}\sum_{r}n_r^*\|c_{\pi(r)}-c_r^*\|^2$, which completes the proof.
\end{proof}
\begin{lm}[Centroidal property, Lemma 2.1 of \cite{kanungo}]
\label{centroidal}
For any point set $Y$ and any point $c$ in $\mathbb{R}^d$, 
$\phi(c,Y) = \phi(m(Y),Y)+|Y|\|m(Y)-c\|^2$.
\end{lm}
\begin{prop}
\label{prop:geom}
Assume (B1) (B2) hold for a stationary clustering $A^*$ with corresponding $k$ centroids denoted by $C^*$.
Then, for any $C$ such that 
$\Delta(C,C^*)\le b\phi^*$ for some $b\le \frac{\gamma^2 f(\alpha)^2}{16^2}$,
we have
$
\max_{r\in [k]}\frac{|A_r\triangle A_r^*|}{n_r^*}
\le
\frac{b}{\gamma f(\alpha)^3}
$.
And $C^*$ is a 
$
(\frac{\gamma^2 f(\alpha)^2}{16^2}, \alpha)
$ 
-stable stationary point.
\end{prop}
\begin{proof}
The condition implies for all $r\in [k]$,
$
\|c_r-c_r^*\|\le \sqrt{\frac{b\phi^*}{n_r^*}}$. 
Then for all $r\ne s$,
\begin{eqnarray*}
\|c_r-c_r^*\|+\|c_s-c_s^*\|
\le\sqrt{b}\sqrt{\phi^*}(\frac{1}{\sqrt{n_r^*}}+\frac{1}{\sqrt{n_s^*}})
=\frac{\sqrt{b}}{\gamma f(\alpha)}\gamma f(\alpha)\sqrt{\phi^*}(\frac{1}{\sqrt{n_r^*}}+\frac{1}{\sqrt{n_s^*}})\\
\le
\frac{\sqrt{b}}{\gamma f(\alpha)}\Delta_{rs}
\le
\frac{1}{16}\Delta_{rs}
\end{eqnarray*} 
where the second inequality is by assumptions (B1) and (B2),
and the last inequality by our assumption on $b$.
Thus, we may apply Lemma \ref{lm:kumar} to get
$
\frac{|A_{r}\triangle A_r^*|}{n_r^*}\le \frac{b}{\gamma^2 f^3(\alpha)}
$ for all $r$, proving the first statement.
Now by Lemma \ref{lm:kmdist_cdist},
$\phi(C)\le (b+1)\phi^*$, so we get
\begin{eqnarray*}
\frac{\alpha b}{5\alpha b+4(1+\frac{\phi(C)}{\phi^*})}
\ge  
\frac{\alpha b}{5\alpha b+4(2+b)}\\
\ge
\frac{\alpha b}{5\alpha \gamma^2f^2(\alpha)/16^2+4(2+\gamma^2f^2(\alpha)/16^2)}
\ge \frac{b}{\gamma^2 f^3(\alpha)}
\ge \frac{|A_{r}\triangle A_r^*|}{n_r^*}
\end{eqnarray*} 
where the last inequality holds
since
$f(\alpha)\ge \frac{5\alpha+5}{16^2\alpha}
$ and $\gamma \ge \frac{8\sqrt{2}}{\sqrt{f}}$ by (B1) and (B2), respectively. 
This proves the second statement since $C^*$ is then $(\frac{\gamma^2 f(\alpha)^2}{16^2}, \alpha)$-stable 
by Definition \ref{defn:stable}.
\end{proof}
\subsection{Proofs regarding seeding guarantee}
\begin{lm}[Theorem 4 of \cite{tang_montel:aistats16}]
\label{lm:adapt_tang}
Assume (B1) and (B2) hold for a stationary clustering $A^*$ with $C^*:=m(A^*)$. If we obtain seeds from Algorithm \ref{alg:seeding}, then
$$
\Delta(C^0,C^*)\le \frac{1}{2}\frac{\gamma^2f(\alpha)^2}{16^2}\phi^*
$$
with probability at least
$
1-m_o\exp(-2(\frac{f(\alpha)}{4}-1)^2w_{\min}^2)-k \exp (-m_op_{\min})
$.
\end{lm}
\begin{proof}
First note that assumption (B1) satisfies center-separability assumption in Definition 1 of \cite{tang_montel:aistats16}. Therefore, applying Theorem 4 of \cite{tang_montel:aistats16} with $\mu_r=c_r^*$ and $\nu_r=c_r^0$, we get
$\forall r\in [k]$,
$
\|c_r^0-c_r^*\|\le \frac{\sqrt{f(\alpha)}}{2}\sqrt{\frac{\phi_r^*}{n_r^*}}
$
with probability at least
$
1-m_o\exp(-2(\frac{f(\alpha)}{4}-1)^2w_{\min}^2)-k \exp (-m_op_{\min})
$.
Summing over all $r$, the previous event implies
$
\sum_{r}n_r^*\|c_r^0-c_r^*\|^2\le \frac{f(\alpha)}{4}\phi^*
\le \frac{1}{2}\frac{\gamma^2f(\alpha)^2}{16^2}\phi^*
$, where the second inequality is by $\gamma \ge \frac{8\sqrt{2}}{\sqrt{f}}$.
\end{proof}
\begin{lm}
\label{lm:augmented}
Assume the conditions Lemma \ref{lm:adapt_tang} hold. For any $\xi>0$, if 
in addition, 
$$
f(\alpha)\ge 5\sqrt{\frac{1}{2w_{\min}}\ln (\frac{2}{\xi p_{\min}}\ln\frac{2k}{\xi})}
$$
If we obtain seeds from Algorithm \ref{alg:seeding} choosing 
$$
\frac{\ln \frac{2k}{\xi}}{p_{\min}}<m_o<\frac{\xi}{2}\exp\{2(\frac{f(\alpha)}{4}-1)^2w_{\min}^2\}
$$ 
Then
$\Delta(C^0,C^*)\le \frac{1}{2}\frac{\gamma^2f(\alpha)^2}{16^2}\phi^*$
with probability at least
$
1-\xi
$.
\end{lm}
\begin{proof}
By Lemma \ref{lm:adapt_tang},
a sufficient condition for the success probability to be at least $1-\xi$ is:
$$
m_o\exp(-2(\frac{f(\alpha)}{4}-1)^2w^2_{\min})\le \frac{\xi}{2}
$$ and 
$$
k\exp(-m_o p_{\min})\le \frac{\xi}{2}
$$ 
This translates to requiring
$$
\frac{1}{p_{\min}}\ln \frac{2k}{\xi}\le m_o\le  \frac{\xi}{2}\exp(2(\frac{f(\alpha)}{4}-1)^2w^2_{\min})
$$
Note for this inequality to be possible, we also need
$
\frac{1}{p_{\min}}\ln \frac{2k}{\xi}\le  \frac{\xi}{2}\exp(2(\frac{f(\alpha)}{4}-1)^2w^2_{\min})
$, 
imposing a constraint on $f(\alpha)$. 
Taking logarithm on both sides and rearrange, we get
$$
(\frac{f(\alpha)}{4}-1)^2\ge \frac{1}{2w_{\min}}\ln (\frac{2}{\xi p_{\min}}\ln\frac{2k}{\xi})
$$
That is,
$
f(\alpha)\ge 5\sqrt{\frac{1}{2w_{\min}}\ln (\frac{2}{\xi p_{\min}}\ln\frac{2k}{\xi})}
$.
\end{proof}
\begin{proof}[Proof of Theorem \ref{thm:solution}]
We first show clusterability of the dataset
implies stability of the optimal solution:
since $C^{opt}$, which is necessarily a stationary point, satisfies (B1)(B2),
$C^{opt}$ is a $(\frac{\gamma f(\alpha)^2}{16^2},\alpha)$-stable stationary point by Proposition \ref{prop:geom}.
Let $b_0:=\frac{\gamma f(\alpha)^2}{16^2}$,
and we denote event
$
F:=\{\Delta(C^0,C^{opt})\le \frac{1}{2}b_0\phi^{opt}\}
$.
Since 
$
f(\alpha)\ge 5\sqrt{\frac{1}{2w_{\min}}\ln (\frac{2}{\xi p_{\min}}\ln\frac{2k}{\xi})}
$, 
and 
$
\frac{\log\frac{2k}{\xi}}{p_{\min}}
<m_o
<\frac{\xi}{2}\exp\{2(\frac{f(\alpha)}{4}-1)^2w_{\min}^2\}
$, we can apply
Lemma \ref{lm:augmented} to get
$$
Pr\{F\}\ge 1-\xi
$$ 
Conditioning on $F$, we can invoke Theorem \ref{thm:mbkm_local},
with $c^{\prime},t_o$ sufficiently large;
both depends on $p_r^t(m)$ as defined in Theorem \ref{thm:mbkm_local}.
Since $C^{opt}$ is $(b_0,\alpha)$-stable, conditioning on $F$
$$
\max_r\frac{A_r^t\triangle A_r^{opt}}{n_r^{opt}}
\le 
\frac{b_0}{\gamma f(\alpha)^3}
\le
\frac{\gamma}{16^2 f(\alpha)}
$$
So
$$
\min_{r,t} p_r^t(1)|F \ge p_{\min} - \max_r\frac{A_r^t\triangle A_r^{opt}}{n_r^{opt}}
\ge p_{\min} - \frac{\gamma}{16^2 f(\alpha)}
$$
And 
$$
\min_{r,t} p_r^t(m)|F
\ge 1 - [1-(p_{\min} - \frac{\gamma}{16^2 f(\alpha)})]^m
$$
where $p_{\min} - \frac{\gamma}{16^2 f(\alpha)}>\sqrt{\alpha}$ by (B3).
To apply Theorem \ref{thm:mbkm_local}, it is sufficient to have
\begin{eqnarray*}
\beta := 2c^{\prime}\min_{r,t}p_r^t(m)(1-\max_{t}a^t\sqrt{\alpha}))
\ge 2c^{\prime}(\min_{r,t}p_r^t(m)-\sqrt{\alpha})\\
\ge
2c^{\prime}(1-\sqrt{\alpha}- [1-(p_{\min} - \frac{\gamma}{16^2 f(\alpha)})]^m)
>
2c^{\prime}(1-\sqrt{\alpha}- [1-\sqrt{\alpha}]^m)
>1
\end{eqnarray*}
by our requirement that $m>1$, and $c^{\prime}>\frac{1}{2[1-\sqrt{\alpha}-(1-\sqrt{\alpha})^m]}$.
\begin{eqnarray*}
t_0
\ge 
\max\left\{
\substack{
(\frac{16(c^{\prime})^2B}{(1 - [1-(p_{\min} - \frac{\gamma}{16^2 f(\alpha)})]^m)^2\Delta^0})^{\frac{1}{\beta-1}},\\
[\frac{48(c^{\prime})^2B^2}{(1 - [1-(p_{\min} - \frac{\gamma}{16^2 f(\alpha)})]^m)^2(\beta-1)(\Delta^0)^2}\ln\frac{1}{\delta}]
^{\frac{2}{\beta-1}},(\ln\frac{1}{\delta})^{\frac{2}{\beta-1}}
}\right\}\\
\ge 
\max\{
(\frac{16a_{\max}^2B}{\Delta^0})^{\frac{1}{\beta-1}},
[\frac{48a_{\max}^2B^2}{(\beta-1)(\Delta^0)^2}\ln\frac{1}{\delta}]^{\frac{2}{\beta-1}},
(\ln\frac{1}{\delta})^{\frac{2}{\beta-1}}
\}
\end{eqnarray*}
with
$
a_{\max}:=\frac{c^{\prime}}{\min_{r,t}p_r^t(m)}
$.
Both are satisfied by our requirement on $c^{\prime}$ and $t_o$, so applying Theorem \ref{thm:mbkm_local} we get
$\forall t\ge 1$,
$$
E\{\Delta^t|\Omega_t,F\}=E\{\Delta^t|\Omega_t\}=O(\frac{1}{t})
$$
where $\Omega_t:=\{\Delta^{t-1} \le b_0\phi^{opt}\}$
and 
$
Pr\{\Omega_t|F\}\ge 1-\delta
$.
So
$$
Pr\{\Omega_t,F\}
=Pr\{\Omega_t|F\}Pr\{F\}
\ge (1-\delta)(1-\xi)
$$
Finally, using Lemma \ref{lm:kmdist_cdist}, and letting $G_t:=\Omega_t\cap F$, we get
the desired result.
\end{proof}
\section{Appendix D: technical lemmas}
\begin{lm}\label{techlm:avgAlg}
Let $w_t, g_t$ denote vectors of dimension $\mathbb{R}^{d}$ at time $t$. If we choose $w_0 $ arbitrarily, and for $t=1\dots T$, we repeatdly apply the following update
$$
w_t = (1-\frac{1}{t})w_{t-1}+\frac{1}{t}g_t
$$
Then
$$
w_T = \frac{1}{T}\sum_{t=1}^{T}g_t
$$ 
\end{lm}
\begin{proof}
We prove by induction on $T$. 
For $T=1$, $w_1 = (1-1)w_0+g_1 = \frac{1}{1}\sum_{t=1}^{1}g_t$. So the claim holds for $T=1$.

Suppose the claim holds for $T$, then for $T+1$, by the update rule
\begin{eqnarray*}
w_{T+1}=(1-\frac{1}{T+1})w_T + \frac{1}{T+1}g_{T+1}\\
=(1-\frac{1}{T+1}) \frac{1}{T}\sum_{t=1}^{T}g_t + \frac{1}{T+1}g_{T+1}\\
=\frac{T}{T+1}\frac{1}{T}\sum_{t=1}^{T}g_t+ \frac{1}{T+1}g_{T+1}\\
=\frac{1}{T+1}\sum_{t=1}^{T+1}g_t
\end{eqnarray*}
So the claim holds for any $T\ge 1$.
\end{proof}
\begin{lm}[technical lemma]
\label{lm:tech}
For $\beta\in (1,2]$.
If $C\ge\frac{\beta-1}{3},\delta\le\frac{1}{e}$, and $t\ge (\frac{3C}{\beta-1}\ln\frac{1}{\delta})^{\frac{2}{\beta-1}}$, then
$
t^{\beta-1}-2C\ln t - C\ln\frac{1}{\delta}> 0
$. 
\end{lm}
\begin{proof}
Let $f(t):=t^{\beta-1}-2C\ln t - C\ln\frac{1}{\delta}$.
Taking derivative, we get
$f^{\prime}(t)=(\beta-1)t^{\beta-2}-\frac{2C}{t}\ge0$ when $t\ge (\frac{2C}{\beta-1})^{\frac{1}{\beta-1}}$. 
Since $\ln\frac{1}{\delta}\frac{3C}{\beta-1}\ge\frac{3C}{\beta-1}\ge1$,
$(\ln\frac{1}{\delta}\frac{3C}{\beta-1})^{\frac{2}{\beta-1}}\ge(\frac{2C}{\beta-1})^{\frac{1}{\beta-1}}$,
it suffices to show $f((\ln\frac{1}{\delta}\frac{3C}{\beta-1})^{\frac{2}{\beta-1}})> 0$ for our statement to hold.
$f((\ln\frac{1}{\delta}\frac{3C}{\beta-1})^{\frac{2}{\beta-1}})
=(\ln\frac{1}{\delta}\frac{3C}{\beta-1})^{2}
-2C\ln\{(\ln\frac{1}{\delta}\frac{3C}{\beta-1})^{\frac{2}{\beta-1}}\}-C\ln\frac{1}{\delta}
=(\ln\frac{1}{\delta})^{2}\frac{9C^2}{(\beta-1)^2}
-\frac{4C}{\beta-1}\ln(\ln\frac{1}{\delta}\frac{3C}{\beta-1})
-C\ln\frac{1}{\delta}
=\frac{4C}{\beta-1}[\frac{\frac{3}{2}C}{\beta-1}\ln\frac{1}{\delta}
-\ln(\frac{3C}{\beta-1}\ln\frac{1}{\delta})]
+C\ln\frac{1}{\delta}[\frac{3C}{(\beta-1)^2}-1]
>0
$,
where the first term is greater than zero because $x-\ln(2x)>0$ for $x>0$,
and the second term is greater than zero by our assumption on $C$.
\end{proof}
\begin{lm}[Lemma D1 of \cite{balsubramani13}]
\label{lm:tech_2}
Consider a nonnegative sequence $(u_t: t\ge t_o)$, such that for some constants $a,b>0$
and for all $t>t_{o}\ge 0$,
$u_t\le (1-\frac{a}{t})u_{t-1}+\frac{b}{t^2}$.
Then, if $a>1$,
$$
u_t\le (\frac{t_o+1}{t+1})^a u_{t_o} + \frac{b}{a-1}(1+\frac{1}{t_o+1})^{a+1}\frac{1}{t+1}
$$
\end{lm}
\end{document}